\def\final{0}

\documentclass{article}

\usepackage{fullpage,graphicx}
\usepackage{algorithm, algorithmic}
\usepackage{amsmath, amssymb, amsthm}
\usepackage{dsfont}
\usepackage{enumerate}
\usepackage{enumitem}
\usepackage{framed}
\usepackage{verbatim}
\usepackage{color}
\usepackage{microtype}
\usepackage{kpfonts}
\usepackage{bbold,natbib}

\DeclareMathAlphabet{\mathsf}{OT1}{cmss}{m}{n}
	\SetMathAlphabet{\mathsf}{bold}{OT1}{cmss}{bx}{n}

\usepackage{multirow, tabularx}
\definecolor{DarkGreen}{rgb}{0.15,0.5,0.15}
\definecolor{DarkRed}{rgb}{0.6,0.2,0.2}
\definecolor{DarkBlue}{rgb}{0.15,0.15,0.55}
\definecolor{DarkPurple}{rgb}{0.4,0.2,0.4}

\definecolor{gray}{gray}{0.4}
\newcommand{\gray}[1]{{\textcolor{gray}{#1}}}

\usepackage[pdftex]{hyperref}
\hypersetup{
    linktocpage=true,
    colorlinks=true,				%
    linkcolor=DarkBlue,		%
    citecolor=DarkBlue,	%
    urlcolor=DarkBlue,		%
}

\ifnum\final=0
\newcommand{\mynote}[2]{{\color{#1} \marginpar{\tiny #2}}}
\newcommand{\mybignote}[2]{{\color{#1} $\langle \langle$ #2$\rangle \rangle$}}
\else
\newcommand{\mynote}[2]{}
\newcommand{\mybignote}[2]{}
\fi

\newcolumntype{Y}{>{\centering\arraybackslash}X}

\newcommand{\eps}{\varepsilon}

\newcommand{\N}{\mathbb{N}}

\newcommand{\AAA}{\mathcal A}
\newcommand{\BBB}{\mathcal B}

\newcommand{\DDD}{\mathcal D}

\newcommand{\error}{{\rm error}}
\newcommand{\VC}{{\rm VC}}

\newtheorem{theorem}{Theorem}[section]

\newtheorem{lemma}[theorem]{Lemma}
\newtheorem{lem}[theorem]{Lemma}

\newtheorem{claim}[theorem]{Claim}

\newtheorem{remark}[theorem]{Remark}

\newtheorem{prop}[theorem]{Proposition}

\newtheorem{obs}{Observation}

\theoremstyle{definition}

\newtheorem{definition}[theorem]{Definition}

\newcommand{\Lap}{\operatorname{Lap}}

\newcommand{\nope}[1]{}

\newcommand{\lT}{t_\ell}
\newcommand{\uT}{t_u}
\newcommand{\hlT}{\hat{t}_\ell}
\newcommand{\huT}{\hat{t}_u}
\newcommand{\leftsymb}{\operatorname{L}}
\newcommand{\rightsymb}{\operatorname{R}}
\newcommand{\haltsymb}{\top}

\newcommand{\remove}[1]{}

\makeatletter
\renewcommand{\fnum@algorithm}{\fname@algorithm}
\makeatother
\makeatletter
\newlength{\fboxhsep}
\newlength{\fboxvsep}
\newlength{\fboxtoprule}
\newlength{\fboxbottomrule}
\newlength{\fboxleftrule}
\newlength{\fboxrightrule}
\def\@frameb@xother#1{%
  \@tempdima\fboxtoprule
  \advance\@tempdima\fboxvsep
  \advance\@tempdima\dp\@tempboxa
  \hbox{%
    \lower\@tempdima\hbox{%
      \vbox{%
        \hrule\@height\fboxtoprule
        \hbox{%
          \vrule\@width\fboxleftrule
          #1%
          \vbox{%
            \vskip\fboxvsep
            \box\@tempboxa
            \vskip\fboxvsep}%
          #1%
          \vrule\@width\fboxrightrule}%
        \hrule\@height\fboxbottomrule}%
    }%
  }%
}
\long\def\fboxother#1{%
  \leavevmode
  \setbox\@tempboxa\hbox{%
    \color@begingroup
    \kern\fboxhsep{#1}\kern\fboxhsep
    \color@endgroup}%
  \@frameb@xother\relax}

\makeatother

\title{Private Everlasting Prediction}
\author{Moni Naor\thanks{Department of Computer Science and Applied Math, Weizmann Institute of Science {\tt moni.naor@weizmann.ac.il}. Incumbent of the Judith Kleeman Professorial Chair. Research supported in part by grants from the Israel Science Foundation (no.2686/20), by the Simons Foundation Collaboration on the Theory of Algorithmic Fairness and by the Israeli Council for Higher
Education (CHE) via the Weizmann Data Science Research Center.}
\and Kobbi Nissim\thanks{Department of Computer Science, Georgetown University {\tt kobbi.nissim@georgetown.edu}. Partially supported by NSF grant No. CNS-2001041 and a gift to Georgetown University.}
\and Uri Stemmer\thanks{Blavatnik School of Computer Science, Tel Aviv University, and Google Research. \texttt{u@uri.co.il}. Partially supported by the Israel Science Foundation (grant 1871/19) and by
	Len Blavatnik and the Blavatnik Family foundation.}
\and Chao Yan\thanks{Department of Computer Science, Georgetown University {\tt cy399@georgetown.edu}. Partially supported by a gift to Georgetown University.}
}
\date{May 16, 2023}

\begin{document}

\maketitle

\begin{abstract} 

A private learner is trained on a sample of labeled points and generates a hypothesis that can be used for predicting the labels of newly sampled points while protecting the privacy of the training set [Kasiviswannathan et al., FOCS 2008]. Research uncovered that private learners may need to exhibit significantly higher sample complexity than non-private learners as is the case with, e.g., learning of one-dimensional threshold functions [Bun et al., FOCS 2015, Alon et al., STOC 2019].

We explore prediction as an alternative to learning. Instead of putting forward a hypothesis, a predictor answers a stream of classification queries. 
Earlier work has considered a private prediction model with just a single classification query [Dwork and Feldman, COLT 2018]. We observe that when answering a stream of queries, a predictor must modify the hypothesis it uses over time, and, furthermore, that it must use the queries for this modification, hence introducing potential privacy risks with respect to the queries themselves. 

We introduce {\em private everlasting prediction} taking into account the privacy of both the training set {\em and} the (adaptively chosen) queries made to the predictor. 
We then present a generic construction of private everlasting predictors in the PAC model.
The sample complexity of the initial training sample in our construction is quadratic (up to polylog factors) in the VC dimension of the concept class. Our construction allows prediction for all concept classes with finite VC dimension, and in particular threshold functions with constant size initial training sample, even when considered over infinite domains, whereas it is known that the sample complexity of privately learning threshold functions must grow as a function of the domain size and hence is impossible for infinite domains.
\end{abstract}

\section{Introduction}

A PAC learner is given labeled examples $S=\left\{(x_i,y_i)\right\}_{i\in[n]}$  drawn i.i.d.\ from an unknown underlying probability distribution $\DDD$ over a data domain $X$ and outputs a hypothesis $h$ that can be used for predicting the label of fresh points $x_{n+1},x_{n+2},\ldots$ sampled from the same underlying probability distribution $\DDD$~\citep{Valiant84}. It is well known that when points are labeled by a concept selected from a concept class $C=\left\{c:X\rightarrow \{0,1\}\right\}$ then learning is possible with sample complexity proportional to the VC dimension of the concept class.

Learning often happens in settings where the underlying training data is related to individuals and privacy-sensitive and where a learner is required, for legal, ethical, or other reasons, to protect personal information from being leaked in the learned hypothesis $h$. 
Private learning was introduced by \cite{KLNRS08}, as a theoretical model for studying such tasks. A {\em private learner} is a PAC learner that preserves differential privacy with respect to its training set $S$. That is, the learner's distribution on outcome hypotheses must not depend too strongly on any single example in $S$.
Kasiviswanathan et al.\ showed
via a generic construction 
that any finite concept class can be learned privately and with sample complexity $n=O(\log |C|)$. This value ($O(\log |C|)$)  can be significantly higher than the VC dimension of the concept class $C$ (see below).

It is now understood that the gap between the sample complexity of private and non-private learners is essential -- an important example is private learning of threshold functions (defined over an ordered domain $X$ as $C_{thresh}=\{c_t\}_{t\in X}$ where $c_t(x)=\mathbb{1}_{x\geq t}$), which requires sample complexity that is asymptotically higher than the (constant) VC dimension of $C_{thresh}$. 
In more detail, with {\em pure} differential privacy, the sample complexity of private learning is characterized by the representation dimension of the concept class \citep{BNS13}. 
The representation dimension of $C_{thresh}$ (hence, the sample complexity of private learning thresholds) is $\Theta(\log |X|)$~\citep{FX14}. 
With {\em approximate} differential privacy, the sample complexity of learning threshold functions is $\Theta(\log^*|X|)$~\citep{BNS13b, BNSV15,AlonLMM19,KaplanLMNS20,CohenLNSS22}. 
Hence, in both the pure and approximate differential privacy cases, the sample complexity grows with the cardinality of the domain $|X|$ and no private learner exists for threshold functions over infinite domains, such as the integers and the reals, whereas low sample complexity non-private learners exist for these tasks.

\paragraph{Privacy preserving (black-box) prediction.} 

\cite{DworkF18} proposed privacy-preserving prediction as an alternative for private learning. Noting that ``[i]t is now known that for some basic learning problems [. . .] producing an accurate private model requires much more data than learning without privacy,'' they considered a setting where ``users may be allowed to query the prediction model on their inputs only through an appropriate interface''.
That is, a setting where the learned hypothesis is not made public. Instead, it may be accessed in a ``black-box'' manner via a privacy-preserving query-answering prediction interface. The prediction interface is required to preserve the privacy of its training set $S$:

\begin{definition}[private prediction interface~\citep{DworkF18} (rephrased)] 
\label{def:DworkFeldmanPredictionInterface}
A prediction interface $M$ is $(\epsilon,\delta)$-differentially private if for every interactive query generating algorithm $Q$, the output of the interaction between $Q$ and $M(S)$ is $(\epsilon,\delta)$-differentially private with respect to $S$.
\end{definition}

Dwork and Feldman focused on the setting where the entire interaction between $Q$ and $M(S)$ consists of issuing a single prediction query and answering it:
\begin{definition}[Single query prediction~\citep{DworkF18}]
\label{def:DworkFeldmanPrivatePrediction} Let $M$ be an algorithm that given a set of labeled examples $S$ and an unlabeled point $x$ produces a label $y$. 
$M$ is an  $(\epsilon,\delta)$-differentially private prediction algorithm if for every $x$, the output $M(S,x)$  is $(\epsilon,\delta)$-differentially private with respect to $S$.
\end{definition}

W.r.t.\ answering a single prediction query, Dwork and Feldman showed that the sample complexity of such predictors is proportional to the VC dimension of the concept class. 

\subsection{Our contributions}

In this work, we extend private prediction beyond a single query to answering any sequence -- {\em unlimited in length} -- of prediction queries. We refer to this as {\em private everlasting prediction}. Our goal is to present a generic private everlasting predictor with low training sample complexity $|S|$. 

\paragraph{Private prediction interfaces when applied to a large number of queries.}
We begin by examining private everlasting prediction under the framework of Definition~\ref{def:DworkFeldmanPredictionInterface}. 
We prove:

\begin{theorem}[informal version of Theorem~\ref{thm:SdoesNotSuffice}]\label{thm:SdoesNotSufficeIntro}
Let $\cal A$ be a private everlasting prediction interface for concept class $C$ and assume $\cal A$ bases its predictions solely on the initial training set $S$, then there exists a private learner for concept class $C$ with sample complexity $|S|$.
\end{theorem}

This means that everlasting predictors that base their prediction solely on the initial training set $S$ are subject to the same complexity lowerbounds as private learners.
Hence, to avoid private learning lowerbounds, private everlasting predictors need to rely on more than the initial training sample $S$ as a source of information about the underlying probability distribution and the labeling concept. 

In this work, we choose to allow the everlasting predictor to rely on the queries made - which are unlabeled points from the domain $X$, assuming the queries are drawn from the same distribution the initial training $S$ is sampled from. This requires changing the privacy definition, as Definition~\ref{def:DworkFeldmanPredictionInterface} does not protect the queries made, yet the classification given to a query can now depend on and hence reveal information provided in queries made earlier.

\paragraph{A definition of private everlasting predictors.} Our definition of private everlasting predictors is motivated by the observations above.
Consider an algorithm $\AAA$ that is first fed with a training set $S$ of labeled points and then executes for an unlimited number of rounds, where in round $i$ algorithm $\AAA$ receives as input a query point $x_i$ and produces a label $\hat y_i$. 
We say that $\AAA$ is an everlasting predictor if, when the (labeled) training set $S$ and the (unlabeled) query points are coming from the same underlying distribution, $\AAA$ answers each query points $x_i$ with a good hypothesis $h_i$, and hence the label $\hat y_i$ produced by $\AAA$ is correct with high probability.  
We say that $\AAA$ is a {\em private} everlasting predictor if its sequence of predictions $\hat y_1,\hat y_2, \hat y_3, \ldots$ protects both the privacy of the training set $S$ {\em and} the query points $x_1,x_2,x_3,\ldots$ in face of any adversary that adaptively chooses the query points. 

We emphasize that while private everlasting predictors need to exhibit average-case utility -- as good prediction is required only for the case where $S$ and $x_1,x_2,x_3,\ldots$ are selected i.i.d.\ from the same underlying distribution -- our privacy requirement is worst-case, and holds in face of an {\em adaptive} adversary that chooses each query point $x_i$ after receiving the prediction provided for $(x_1,\ldots, x_{i-1})$, and not necessarily in accordance with any probability distribution.  

\paragraph{A generic construction of private everlasting  predictors.} Our construction, called \texttt{GenericBBL}, executes in rounds. The input to the first round is the initial labeled training set $S$, where the number of samples in $S$ is quadratic in the VC dimension of the concept class. Each other round begins with a collection $S_i$ of labeled examples and ends with newly generated collection of labeled examples $S_{i+1}$. The set $S$ is assumed to be consistent with some concept $c\in C$ and our construction ensures that this is the case also for the sets $S_i$ for all $i$. We briefly describe the main computations performed in each round of \texttt{GenericBBL}.\footnote{Important details, such as privacy amplification via sampling and management of the learning accuracy and error parameters are omitted from the description provided in this section.}
\begin{itemize}[leftmargin=20px]
    \item {\bf Round initialization:} At the outset of a round, the labeled set $S_i$ is partitioned into sub-sets, each with number of samples which is proportional to the VC dimension (so we have $\approx \frac{|S_i|}{\VC(C)}$ sub-sets). Each of the sub-sets is used for training a classifier non-privately, hence creating a collection of classifiers $F_i=\left\{f: X\rightarrow \{0,1\}\right\}$ that are used throughout the round.
    
    \item {\bf Query answering:} Queries are issued to the predictor in an online manner. Each query is first labeled by each of the classifiers in $F_i$. Then the predicted label is computed by applying a privacy-preserving majority vote on these intermediate labels. (By standard composition theorems for differential privacy, we could answer roughly $|F_i|^2\approx\left(\frac{|S_i|}{\VC(C)}\right)^2$ queries without exhausting our privacy budget.) To save on the privacy budget, the majority vote is based on the \texttt{BetweenThresholds} mechanism of \cite{BunSU16} (which in turn is based on the sparse vector technique). The algorithm fails when the privacy budget is exhausted. However, when queries are sampled from the underlying distribution then with a high enough probability the labels produced by the classifiers in $F_i$ would exhibit a clear majority.

    \item {\bf Generating a labeled set for the following round:} The predictions provided in the duration of a round are not guaranteed to be consistent with any concept in $C$ and hence cannot be used to set the following round. Instead, at the end of the round these points are relabeled consistently with $C$ using a technique developed by \cite{BeimelNS21} in the context of private semi-supervised learning. 
    Let $S_{i+1}$ denote the query points obtained during the $i$th round, after (re)labeling them. This is a collection of size $|S_{i+1}|\approx\left(\frac{|S_i|}{\VC(C)}\right)^2$. Hence, provided that $|S_i|\gtrsim\left(\VC(C)\right)^2$ we get that $|S_{i+1}|>|S_i|$ which allows us to continue to the next round with more data than we had in the previous round.
\end{itemize}

We prove:

\begin{theorem}[informal version of Theorem~\ref{thm:privateEverlastingPredictor}]
For every concept class $C$, Algorithm \texttt{GenericBBL} is a private everlasting predictor requiring an initial set of labeled examples which is (upto polylogarithmic factors) quadratic in the VC dimension of $C$.
\end{theorem}

\subsection{Related work}

Beyond the work of \cite{DworkF18} on private prediction mentioned above, our work is related to private semi-supervised learning and joint differential privacy.

\paragraph{Semi-supervised private learning.} As in the model of private semi-supervised learning of \cite{BeimelNS21}, our predictors depend on both labeled and unlabeled sample. Beyond the obvious difference between the models (outputting a hypothesis vs.\ providing black-box prediction), a major difference between the settings is that in the work of \cite{BeimelNS21} all samples -- labeled and unlabeled - are given at once at the outset of the learning process whereas in the setting of everlasting predictors the unlabeled samples are supplied in an online manner. Our construction of private everlasting predictors uses tools developed for the semi-supervised setting, and in particular Algorithm \texttt{LabelBoost} of of Beimel et al. %

\paragraph{Joint differential privacy.} \cite{KearnsPRRU15} introduced joint differential privacy (JDP) as a relaxation of differential privacy applicable for mechanism design and games. For every user $u$, JDP requires that the outputs jointly seen by all other users would preserve differential privacy w.r.t.\ the input of $u$. 
Crucially, in JDP users select their inputs ahead of the computation. In our settings, the inputs to a private everlasting predictor are prediction queries which are chosen in an online manner, and hence a query can depend on previous queries and their answers. Yet, similarly to JDP, the outputs provided to queries not performed by a user $u$ should jointly preserve differential privacy w.r.t.\ the query made by $u$. Our privacy requirement hence extends JDP to an adaptive online setting.

\paragraph{Additional works on private prediction.}
\cite{BassilyTT18} studied a variant of the private prediction problem where the algorithm takes a labeled sample $S$ and is then required to answer $m$ prediction queries (i.e., label a sequence of $m$ unlabeled points sampled from the same underlying distribution). They presented algorithms for this task with sample complexity $|S|\gtrsim\sqrt{m}$. This should be contrasted with our model and results, where the sample complexity is independent of $m$. The bounds presented by \cite{DworkF18} and \cite{BassilyTT18} were improved by
\cite{DaganF20} and by 
\cite{NandiB20} who
presented algorithms with improved dependency on the accuracy parameter in the agnostic setting.

\subsection{Discussion and open problems}

We show how to
transform any (non-private) learner for the class $C$ (with sample complexity proportional to the VC dimension of $C$) to a private everlasting predictor for $C$. Our construction is not polynomial time due to the use of Algorithm \texttt{LabelBoost}, and requires 
an initial set $S$ of labeled examples which is quadratic in the VC dimension. 
We leave open the question whether $|S|$ can be reduced to be linear in the VC dimension and whether the construction can be made polynomial time. A few remarks are in order:

\begin{enumerate}[leftmargin=20px,topsep=0px]

    \item Even though our generic construction is not computationally efficient, it does result in efficient learners for several interesting special cases. Specifically, algorithm \texttt{LabelBoost} can be implemented efficiently whenever given an input sample $S$ we could efficiently enumerate all possible dichotomies from the target class $C$ over the points in $S$. In particular, this is the case for the class of 1-dim threshold functions $C_{thresh}$, as well as additional classes with constant VC dimension. Another notable example is the class $C^{enc}_{thresh}$ which intuitively is an ``encrypted'' version of  $C_{thresh}$. 
    \cite{BunZ16} showed that (under plausible cryptographic assumptions) the class $C^{enc}_{thresh}$ cannot be learned privately and efficiently, while non-private learning is possible efficiently. Our construction can be implemented efficiently for this class. This provides an example where private everlasting prediction can be done efficiently, while (standard) private learning is possible but inefficient.

    \item It is now known that some learning tasks require the produced model to memorize parts of the training set in order to achieve good learning rates, which in particular disallows the learning algorithm from satisfying (standard) differential privacy \citep{BrownBFST21}. Our notion of private everlasting prediction circumvents this issue, since the model is never publicly released and hence the fact that it must memorize parts of the sample is not of a direct privacy threat. In other words, our work puts forward a private learning model which, in principle, allows memorization. This could have additional applications in broader settings.

    \item As we mentioned, in general, private everlasting predictors cannot base their predictions solely on the initial training set, and in this work we choose to rely on the {\em queries} presented to the algorithm (in addition to the training set). Our construction can be easily adapted to a setting where the content of the blackbox is updated based on {\em fresh unlabeled samples} (whose privacy would be preserved), instead of relying on the query points themselves. This might be beneficial in order to avoid poisoning attacks via the queries.
\end{enumerate}

\section{Preliminaries}
\subsection{Preliminaries from differential privacy}

\begin{definition}[$(\epsilon,\delta)$-indistinguishability]
Let $R_0,R_1$ be two random variables over the same support. We say that $R_0,R_1$ are $(\epsilon,\delta)$-indistinguishable if for every event $E$ defined over the support of $R_0,R_1$,
$$\Pr[R_0\in E]\leq e^\epsilon \cdot\Pr[R_1\in E]+\delta~\text{ and }~ \Pr[R_1\in E]\leq e^\epsilon \cdot \Pr[R_0\in E]+\delta.$$
\end{definition}

\begin{definition}Let $X$ be a data domain. Two datasets $x,x'\in X^n$ are called {\em neighboring} if $|\{i:x_i\not=x'_\}|=1$.
\end{definition}

\begin{definition}[differential privacy~\citep{DMNS06}]
A mechanism $M:X^n\rightarrow Y$ is $(\epsilon,\delta)$-differentially private if $M(x)$ and $M(x')$ are $(\epsilon,\delta)$-indistinguishable for all neighboring $x,x'\in X^n$.
\end{definition}

In our analysis, we use the post-processing and composition properties of differential privacy, that we cite in their simplest form.

\begin{prop}[post-processing]
Let $M_1:X^n\rightarrow Y$ be an $(\epsilon,\delta)$-differentially private algorithm and $M_2:Y\rightarrow Z$ be any algorithm. 
Then the algorithm that on input $x\in X^n$ outputs $M_2(M_1(x))$ is $(\epsilon,\delta)$-differentially private.
\end{prop}

\begin{prop}[composition]
Let $M_1$ be a $(\epsilon_1,\delta_1)$-differentially private algorithm and let $M_2$ be $(\epsilon_2,\delta_2)$-differentially private algorithm. 
Then the algorithm that on input $x\in X^n$ outputs $(M_1(x),M_2(x)$ is $(\epsilon_1+\epsilon_2,\delta_1+\delta_2)$-differentially private.
\end{prop}

\begin{definition}[Exponential mehcanism~\citep{McSherryT07}] \label{def:expmech}
Let $q:X^n\times Y\rightarrow \mathbb{R}$ be a score function defined over data domain $X$ and output domain $Y$. Define $\Delta=\max\left(|q(x,r)-q(x',y)|\right)$ where the maximum is taken over all $y\in Y$ and neighbouring databases $x, x'\in X^n$. 
The exponential mechanism is the $\epsilon$-differentially private mechanism which selects an output $y\in Y$ with probability proportional to $e^{\frac{\epsilon q(x,y)}{2\Delta}}$.
\end{definition}

\begin{claim}[Privacy amplification by sub-sampling~\citep{KLNRS08}]
\label{claim:sub-sampling}
Let $\mathcal{A}$ be an $(\varepsilon',\delta')$-differentially private algorithm operating on a database of size $n$. 
Let $\varepsilon\leq 1$ and let $t=\frac{n}{\varepsilon}(3+\mbox{exp}(\varepsilon'))$. Construct an algorithm $\mathcal{B}$ operating the database $D=(z_i)^t_{i=1}$. Algorithm $\mathcal{B}$ randomly selects a subset $J\subseteq\{1,2,\dots,t\}$ of size $n$, and executes $\mathcal{A}$ on $D_J=(z_i)_{i\in J}$. Then $\mathcal{B}$ is $\left(\varepsilon,\frac{4\varepsilon}{3+\mbox{exp}(\varepsilon')}\delta'\right)$-differentially private.
\end{claim}

\subsection{Preliminaries from PAC learning}

A concept class $C$ over data domain $X$ is a set of predicates $c:X\rightarrow \{0,1\}$ (called concepts) which label points of the domain $X$ by either 0 or 1.  
A learner $\cal A$ for concept class $C$ is given $n$ examples sampled i.i.d.\ from an unknown probability distribution ${\cal D}$ over the data domain $X$ and  labeled according to an unknown target concept $c\in C$. The learner should output a hypothesis $h:X\rightarrow [0,1]$ that approximates $c$ for the distribution ${\cal D}$. More formally, 
\begin{definition}[generalization error]
The {\em generalization error} of a hypothesis $h:X\rightarrow[0,1]$ with respect to concept $c$ and distribution ${\cal D}$ is defined as 
$\error_{\cal D}(c,h)=\mbox{\rm Exp}_{x \sim {\cal D}}[|h(x)- c(x)|].$
\end{definition}

\begin{definition}[PAC learning~\citep{Valiant84}]
Let $C$ be a concept class over a domain $X$.
Algorithm ${\cal A}$ is an {\em $(\alpha,\beta,n)$-PAC learner} for $C$ if for all $c \in C$ and all distributions ${\cal D}$ on $X$, 
$$\Pr[(x_1,\ldots,x_n)\sim{\cal D}^n \; ; \; h\sim {\cal A}((x_1,c(x_1)),\ldots,(x_n,c(x_n)) \; ; \; \error_{\cal D}(c,h) \leq \alpha] \geq 1-\beta,$$
where the probability is over the sampling of $(x_1,\ldots,x_n)$ from ${\cal D}$ and the coin tosses of ${\cal A}$. The parameter $n$ is the {\em sample complexity} of ${\cal A}$.
\end{definition}

See Appendix~\ref{sec:PrelimsPAC} for additional preliminaries on PAC learning.

\subsection{Preliminaties from private learning}

\begin{definition}[private PAC learning~\citep{KLNRS08}]
Algorithm ${\cal A}$ is a $(\alpha,\beta,\epsilon,\delta,n)$-private PAC learner if (i) ${\cal A}$ is an $(\alpha,\beta,n)$-PAC learner and (ii) ${\cal A}$ is $(\epsilon,\delta)$ differentially private.
\end{definition}

\cite{KLNRS08} provided a generic private learner with $O(\VC(C)\log(|X|)$ labeled samples. \cite{BNS13} introduced the representation dimension and showed that any concept class $C$ can be privately learned with $\Theta(\mbox{RepDim}(C))$ samples.\footnote{We omit the dependency on $\epsilon,\delta,\alpha,\beta$ in this brief review.} 
For the sample complexity of $(\epsilon,\delta)$-differentially private learning of threshold functions over domain $X$, \cite{BNSV15} give a lower bound of $\Omega(\log^*|X|)$. Recently, \cite{CohenLNSS22} give a (nearly) matching upper bound of $\tilde{O}(\log^*|X|)$.

\section{Towards private everlasting prediction}

In this work, we extend private prediction beyond a single query to answering any sequence -- unlimited in length -- of prediction queries. Our goal is to present a generic private everlasting predictor with low training sample complexity $|S|$. 

\begin{definition}[everlasting prediction]\label{def:everlastingPrediction}
Let $\AAA$ be an algorithm with the following properties: 
\begin{enumerate}
	\item Algorithm $\AAA$ receives as input $n$ labeled examples $S=\{(x_i,y_i)\}_{i=1}^n\in(X\times\{0,1\})^n$ and selects a hypothesis $h_0:X\rightarrow\{0,1\}$.
	\item For round $r\in\mathbb{N}$, algorithm $\AAA$ gets a query, which is an unlabeled element $x_{n+r}\in X$, outputs $h_{r-1}(x_{n+r})$ and selects a hypothesis $h_r:X\rightarrow\{0,1\}$.
\end{enumerate}

We say that $\AAA$ is an \emph{$(\alpha,\beta,n)$-everlasting predictor} for a concept class $C$ over a domain $X$ if the following holds for every concept $c\in C$ and for every distribution $\DDD$ over $X$. If $x_1,x_2,\dots$ are sampled i.i.d.\ from $\DDD$, and the labels of the $n$ initial samples $S$ are correct, i.e., $y_i=c(x_i)$ for $i\in[n]$, then
$$
\Pr\left[\exists r\geq 0 \text{ s.t.\ } \error_{\DDD}(c,h_r)>\alpha \right]\leq\beta,
$$
where the probability is over the sampling of $x_1,x_2,\ldots$ from $\DDD$ and the randomness of $\AAA$.
\end{definition}

\begin{definition}\label{def:everlastingPredictionInterface}
    An algorithm $\AAA$ is an $(\alpha,\beta,\epsilon,\delta,n)$-everlasting differentially private prediction interface if (i) $\AAA$ is a $(\epsilon,\delta)$-differentially private prediction interface $M$ (as in Definition~\ref{def:DworkFeldmanPredictionInterface}), and (ii) $\AAA$ is an $(\alpha,\beta,n)$-everlasting predictor.
\end{definition}

As a warmup, consider an $(\alpha,\beta,\epsilon,\delta,n)$- everlasting differentially private prediction interface $\AAA$ for concept class $C$ over (finite) domain $X$ (as in Definition~\ref{def:everlastingPredictionInterface} above). Assume that $\AAA$ does not vary its hypotheses, i.e. (in the language of Definition~\ref{def:everlastingPrediction}) $h_r=h_0$ for all $r>0$.\footnote{Formally, $\AAA$ can be thought of as two mechanisms $(M_0,M_1)$ where $M_0$ is $(\epsilon,\delta)$-differentially private. (i) On input a labeled training sample $S$ mechanism $M_0$ computes a hypothesis $h_0$. (ii) On a query $x\in X$ mechanism $M_1$ replies $h_0(x)$.}
Note that a computationally unlimited adversarial querying algorithm can recover the hypothesis $h_0$ by issuing all queries $x\in X$. Hence, in using $\AAA$ indefinitely we lose any potential benefits to sample complexity of restricting access to $h_0$ to being black-box and getting to the point where the lower-bounds on $n$ from private learning apply.
A consequence of this simple observation is that a private everlasting predictor cannot answer all prediction queries with a single hypothesis -- it must modify its hypothesis over time as it processes new queries. 

We now take this observation a step further, showing that a private everlasting predictor that answers prediction queries solely based on its training sample $S$ is subject to the same sample complexity lowerbounds as private learners.

Consider an $(\alpha,\beta<1/8,\epsilon,\delta,n)$-everlasting differentially private prediction interface $\AAA$ for concept class $C$ over (finite) domain $X$ that upon receiving the training set $S\in\left(X\times \{0,1\}\right)^n$ selects an infinite sequence of hypotheses $\{h_r\}_{r\geq 0}$ where $h_r:X\rightarrow \{0,1\}$. Formally, we can think of $\AAA$ as composed of three mechanisms $\AAA=(M_0,M_1,M_2)$ where $M_0$ is $(\epsilon,\delta)$-differentially private:
\begin{itemize} 
\item On input a labeled training sample $S\in\left(X\times\{0,1\}\right)^n$ mechanism $M_0$ computes an initial state and an initial hypothesis $(\sigma_0,h_0)=M_0(S)$. 
\item On a query $x_{n+r}$ mechanism $M_1$ produces an answer $M_1(x_{n+r})=h_i(x_{n+r})$ and mechanism $M_2$ updates the hypothesis-state pair $(h_{r+1},\sigma_{r+1})=M_2(\sigma_r)$.
\end{itemize} 
Note that as $M_0$ and $M_2$ do not receive the sequence $\{x_{n+r}\}_{r\geq 0}$ as input, the sequence $\{h_r\}_{r\geq 0}$ depends solely on $S$. Furthermore as $M_1$ and $M_2$ post-process the outcome of $M_0$, i.e., the sequence of queries and predictions $\{(x_r,h_r(x_r))\}_{r\geq 0}$ preserves $(\epsilon,\delta)$-differential privacy with respect to the training set $S$. In Appendix~\ref{appendix:proof:SdoesNotSuffice} we prove:

\begin{theorem} \label{thm:SdoesNotSuffice} $\AAA$ can be transformed into a $\left(O(\alpha),O(\beta),\epsilon,\delta,O(n \log (1/\beta)\right)$-private PAC learner for $C$.
\end{theorem}

\subsection{A definition of private everlasting prediction}

Theorem~\ref{thm:SdoesNotSuffice} requires us to seek private predictors whose prediction relies on more information than what is provided by the initial labeled sample. Possibilities include requiring the input of additional labeled or unlabeled examples during the lifetime of the predictor, while protecting the privacy of these examples. In this work we choose to rely on the queries for updating the predictor's internal state. This introduces a potential privacy risk for these queries as sensitive information about a query may be leaked in the predictions following it. Furthermore, we need take into account that a privacy attacker may choose their queries adversarially and adaptively.

\begin{definition}[private everlasting black-box prediction]
An algorithm $\AAA$ is an $(\alpha,\beta,\eps,\delta,n)$-private everlasting black-box predictor for a concept class $C$ if 
\begin{enumerate}
	\item {\bf Prediction:} $\AAA$ is an $(\alpha,\beta,n)$-everlasting predictor for $C$ (as in Definition~\ref{def:everlastingPrediction}).  
	\item {\bf Privacy:} For every adversary $\BBB$ and every $t\geq 1$, the random variables $\mbox{View}_{\BBB,t}^0$ and $\mbox{View}_{\BBB,t}^1$ (defined in Figure~\ref{fig:AdversarialExperiment}) are $(\eps,\delta)$-indistinguishable.
\end{enumerate}

\end{definition}

\begin{figure}[h!]
\begin{framed}
{\bf Parameters:} $b\in\{0,1\}$, $t\in\mathbb{N}$. \\ 

{\bf Training Phase:}

\begin{enumerate}

\item The adversary $\BBB$ chooses two sets of $n$ labeled elements $(x_1^0,y_1^0),\dots,(x_n^0,y_n^0)$ and $(x_1^1,y_1^1),\dots,(x_n^1,y_n^1)$, subject to the restriction $\left|\left\{i\in[n]: (x_i^0,y_i^0)\neq(x_i^1,y_i^1)\right\}\right|\in\{0,1\}$. 

\item If $\exists i$ s.t.\ $(x_i^0,y_i^0)\neq(x_i^1,y_i^1)$ then set ${\rm Flag}=1$. Otherwise set ${\rm Flag}=0$.

\item Algorithm $\AAA$ gets $(x_1^b,y_1^b),\dots,(x_n^b,y_n^b)$ and selects a hypothesis $h_0:X\rightarrow\{0,1\}$. \\ \textbackslash * the adversary $\BBB$ does not get to see the hypothesis $h_0$ *\textbackslash
\end{enumerate}

{\bf Prediction phase:}
\begin{enumerate}[resume]
\item For round $r=1,2,\dots,t$:
\begin{enumerate}
	\item If ${\rm Flag}=1$ then the adversary $\BBB$ chooses two elements $x_{n+r}^0=x_{n+r}^1\in X$. Otherwise, the adversary $\BBB$ chooses two elements $x_{n+r}^0,x_{n+r}^1\in X$. 
 
    \item If $x_{n+r}^0\neq x_{n+r}^1$ then ${\rm Flag}$ is set to 1.
	\item \label{experiment: predictionStep} If $x_{n+r}^0=x_{n+r}^1$ then the adversary $\BBB$ gets $h_{r-1}(x_{n+r}^b)$. \\ \textbackslash * the adversary $\BBB$ does not get to see the label if $x_{n+r}^0\not=x_{n+r}^1$ *\textbackslash

    \item Algorithm $\AAA$ gets $x_{n+r}^b$ and selects a hypothesis $h_r:X\rightarrow\{0,1\}$.  \\ \textbackslash * the adversary $\BBB$ does not get to see the hypothesis $h_r$ *\textbackslash
\end{enumerate}

Let $\mbox{View}_{\BBB,t}^b$ be  $\BBB$'s entire view of the execution, i.e., the adversary's randomness and the sequence of predictions in Step~\ref{experiment: predictionStep}.
\end{enumerate}
\end{framed}
\caption{Definition of $\mbox{View}_{\BBB,t}^0$ and $\mbox{View}_{\BBB,t}^1$.\label{fig:AdversarialExperiment}}
\end{figure}

\section{Tools from prior works}

We briefly describe tools from prior works that we use in our construction. See Appendix~\ref{appendix:tools} for a more detailed account.

\paragraph{Algorithm \texttt{LabelBoost}~\citep{BeimelNS21}:} Algorithm \texttt{LabelBoost} takes as input a partially labeled database $S\circ T\in\left(X\times\{0,1,\bot\}\right)^*$ (where the first portion of the database, $S$, contains labeled examples) and outputs a similar database where both $S$ and $T$ are (re)labeled. 
We use the following lemmata from \cite{BeimelNS21}:

\begin{lem}[privacy of Algorithm \texttt{LabelBoost}]\label{lemma:LabelBoostPrivacy}
Let $\AAA$ be an $(\epsilon,\delta)$-differentially private algorithm operating on labeled databases.
Construct an algorithm $\BBB$ that on input a {\em partially} labeled database $S{\circ}T\in(X\times\{0,1,\bot\})^*$ applies $\AAA$ on the outcome of $\texttt{LabelBoos}(S{\circ}T)$.
Then, $\BBB$ is $(\epsilon+3,4e\delta)$-differentially private.
\end{lem}

\begin{lemma}[Utility of Algorithm \texttt{LabelBoost}]\label{lem:LabelBoostUtility}
Fix $\alpha$ and $\beta$, and let $S{\circ}T$ be s.t.\ $S$ is labeled by some target concept $c\in C$, and s.t.\ 
$|T|\leq\frac{\beta}{e} \VC(C)\exp(\frac{\alpha |S|}{2\VC(C)})-|S|.$
Consider the execution of \texttt{LabelBoost} on $S{\circ}T$, and let $h$ denote the hypothesis chosen by \texttt{LabelBoost} to relabel $S{\circ}T$. 
With probability at least $(1-\beta)$ we have that $\error_{S}(h)\leq\alpha$.
\end{lemma}

\paragraph{Algorithm \texttt{BetweenThresholds}~\citep{BunSU16}:} Algorithm \texttt{BetweenThresholds} takes as input a database $S\in X^n$ and thredholds $t_\ell, t_u$. It applies the sparse vector technique to answer noisy threshold queries with $L$ (below threshold) $R$ (above threshold) and $\top$ (halt). We use the following lemmata by ~\cite{BunSU16} and observe that, using standard privacy amplification theorems, Algorithm \texttt{BetweenThresholds} can be modified to allow for $c$ times of outputting $\top$ before halting, with a (roughly) $\sqrt{c}$ growth in its privacy parameter.

\begin{lem}[Privacy for \texttt{BetweenThresholds}] \label{lem:bt-privacy}
Let $\eps,\delta \in (0,1)$ and $n \in \N$. 
Then algorithm \texttt{BetweenThresholds} is $(\eps, \delta)$-differentially private for any adaptively-chosen sequence of queries as long as the gap between the thresholds $\lT, \uT$ satisfies
$\uT - \lT \ge \frac{12}{\eps n}\left( \log (10/\eps) + \log(1/\delta) + 1\right).$
\end{lem}

\begin{lem}[Accuracy of \texttt{BetweenThresholds}] \label{lem:bt-accuracy}
Let $\alpha, \beta,\eps,\lT,\uT \in (0,1)$ and $n,k \in \N$ satisfy $n \geq \frac{8}{\alpha \eps}\left(\log(k+1) + \log(1/\beta)\right).$ Then, for any input $x \in {X}^n$ and any adaptively-chosen sequence of queries $q_1, q_2, \cdots, q_k$, the answers $a_1, a_2, \cdots a_{\leq k}$ produced by \texttt{BetweenThresholds} on input $x$ satisfy the following with probability at least $1-\beta$. For any $j \in [k]$ such that $a_j$ is returned before \texttt{BetweenThresholds} halts,
(i) $a_j = \leftsymb \implies q_j(x) \le \lT + \alpha$,
(ii) $a_j = \rightsymb \implies q_j(x) \ge \uT - \alpha$, and
(iii) $a_j = \haltsymb \implies \lT - \alpha \le q_j(x) \le \uT + \alpha$.
\end{lem}

\begin{obs}\label{bt-obs}
Using standard composition theorems for differential privacy (see, e.g.,~\cite{DRV10}), we can assume that algorithm \texttt{BetweenThresholds} takes another parameter $c$, and halts after $c$ times of outputting $\top$. In this case, the algorithm satisfies $(\eps',2c\delta)$-differential privacy, for $\eps'=\sqrt{2c\ln(\frac{1}{c\delta})}\eps+c\eps(e^\eps-1)$.
\end{obs}

\section{A Generic Construction}

Our generic construction  Algorithm \texttt{GenericBBL} transforms a (non-private) learner for a concept class $C$ into a private everlasting predictor for $C$. The proof of the following theorem follows from Theorem~\ref{thm:accuracy} and Claim~\ref{clm:GenericBBLprivacy} which are proved in Appendix~\ref{appendix:accuracy}.  %

\begin{algorithm}
\caption{\bf \texttt{GenericBBL}}\label{alg:GenericBBL}
{\bf Initial input:} A labeled database $S\in (X\times\{0,1\})^n$ where $n=\frac{8\tau}{\alpha^3\eps^2}\cdot\left(8\VC(C)\log(\frac{26}{\alpha})+4\log(\frac{4}{\beta})\right)^2\cdot\log(\frac{1}{\delta})\cdot\log^2\left(\frac{64\VC(C)\log(\frac{26}{\alpha})+32\log(\frac{4}{\beta})}{\eps\alpha^2\beta\delta}\right)\cdot\left(3+\mbox{exp}(\eps+4)\right).$
\
\begin{enumerate}[leftmargin=20px]%

\item Let $\tau>1.1*10^{10}$. Set $\alpha_1=\alpha/2$, $\beta_1=\beta/2$. 
Define $\lambda_i=\frac{8\VC(C)\log(\frac{13}{\alpha_i})+4\log(\frac{2}{\beta_i})}{\alpha_i}$. \\
/* by Theorem~\ref{thm:VCbound} $\lambda_i$ samples suffice for PAC learning $C$ with parameters $\alpha_i,\beta_i$ */

\item \label{GenericBBL_SubsampleS1} Let $S_1\subseteq S$ be a random subset of size $n\cdot\frac{\eps }{3+\mbox{exp}(\eps+4)}=\frac{\tau\cdot\lambda_i^2\cdot\log(\frac{1}{\delta})\cdot\log^2(\frac{\lambda_i}{\eps\alpha_i\beta_i\delta})}{\alpha_i\eps}$.

\item \label{step:mainloop} Repeat for $i=1,2,3,\ldots$
\begin{enumerate}
\item \label{step:beginloop} Divide $S_i$ into $T_i=\frac{\tau\cdot\lambda_i\cdot\log(\frac{1}{\delta})\cdot\log^2(\frac{\lambda_i}{\eps\alpha_i\beta_i\delta})}{\alpha_i\eps}$ disjoint databases  $S_{i,1},\dots,S_{i,T_i}$ of size $\lambda_i$.

\item\label{step:generatehypothesis} For $t\in[T_i]$ let $f_t\in C$ be a hypothesis minimizing $\error_{S_{i,t}}(\cdot)$. Define $F_i=(f_1,\ldots,f_{T_i})$.

\item \label{GenericBBL_unlabelR} Set $R_i=\frac{25600|S_i|}{\varepsilon}$. Set $t_u=1/2+\alpha_i, t_{\ell}=1/2-\alpha_i$. Set the privacy parameters $\eps'_i=\frac{1}{3\sqrt{c_i\ln(\frac{2}{\delta})}}$ and $\delta'_i=\frac{\delta}{2c_i}$, where $c_i=64\alpha_iR_i$. Instantiate algorithm \texttt{BetweenThresholds} on the database of hypotheses $F_i$ allowing for $c_i=64\alpha_i R_i$  rounds of $\top$ while satisfying $(1,\delta)$-differential privacy (as in Observation~\ref{bt-obs}).

\item \label{GenericBBL_main_loop} For $\ell=1$ to $R_i$:

\begin{enumerate}
	\item Receive as input a prediction query $x_{i,\ell}\in X$.
	\item Give \texttt{BetweenThresholds} the query $q_{x_{i,\ell}}$ where $q_{x_{i,\ell}}(F_i)=\sum_{t\in[T_i]}f_t(x_{i,\ell})$, and obtain an outcome $y_{i,\ell}\in\{L,\top,R\}$.
	\item \label{GenericBBL_label_point} Respond with the label $0$ if $y_{i,\ell}=L$ and $1$ if $y_{i,\ell}\in\{R,\top\}$.
	\item\label{step:betweenThresholdFail} If \texttt{BetweenThresholds} halts, then halt and fail (recall that \texttt{BetweenThresholds} only halts if $c_i$ copies of $\top$ were encountered during the current iteration).
\end{enumerate}

\item \label{GenericBBL_UnlabeledDatabases} Denote $D_i=(x_{i,1},\dots,x_{i,R_i})$.

\item \label{GenericBBL_LabelBoost} Let $\hat{S}_i\subseteq S_i$ and $\hat{D}_i\subseteq D_i$ be random subsets of size $\frac{\eps|S_i|}{3+\mbox{exp}(\eps+4)}$ and $\frac{\eps|D_i|}{3+\mbox{exp}(\eps+4)}$ respectively, and let $\hat{S}_i'{\circ}\hat{D}_i' \leftarrow \texttt{LabelBoost}(\hat{S}_i{\circ}\hat{D}_i)$. Let $S_{i+1}\subseteq \hat{D}_i'$ be a random subset of size $\lambda_{i+1}T_{i+1}$.

\item \label{step:endloop} Set $\alpha_{i+1}\leftarrow\alpha_i/2$ and $\beta_{i+1}\leftarrow\beta_i/2$.

\end{enumerate}
\end{enumerate}
\end{algorithm}

\begin{theorem} \label{thm:privateEverlastingPredictor}
Given $\alpha,\beta,\delta<1/16, \epsilon<1$, Algorithm \texttt{GenericBBL} is a $(6\alpha, 4\beta,\epsilon,\delta,n)$-private everlasting predictor, where $n$ is set as in Algorithm \texttt{GenericBBL}.
\end{theorem}

\begin{theorem}[accuracy of algorithm \texttt{GenericBBL}] \label{thm:accuracy}
Given $\alpha,\beta,\delta<1/16$, $\eps<1$, for any concept $c$ and any round $r$, algorithm \texttt{GenericBBL} can predict the label of $x_r$ as $h_r(x_r)$, such that $\Pr[error_{\DDD}(c(x_r)\neq h_r(x_r))\leq 6\alpha]\geq 1-4\beta$.
\end{theorem}

\begin{claim} \label{clm:GenericBBLprivacy}
\texttt{GenericBBL} is $(\varepsilon,\delta)$-differentially private.
\end{claim}

\begin{remark}
For simplicity, we analyzed \texttt{GenericBBL} in the realizable setting, i.e., under the assumption that the training set $S$ is {\em consistent} with the target class $C$. Our construction carries over to the agnostic setting via standard arguments (ignoring computational efficiency). %
We refer the reader to \citep{BeimelNS21} and \citep{AlonBMS20} for generic agnostic-to-realizable reductions in the context of private learning.
\end{remark}

\newpage

\bibliographystyle{plainnat}

\appendix

\section{Additional Preliminaries from PAC Learning}\label{sec:PrelimsPAC}

It is well know that that a sample of size $\Theta(\mbox{\rm VC}(C))$ is necessary and sufficient for the PAC learning of a concept class $C$, where the Vapnik-Chervonenkis (VC) dimension of a class $C$ is defined as follows:
\begin{definition}[VC-Dimension~\citep{VC}]
Let $C$ be a concept class over a domain $X$. For a set $B=\{b_1,\ldots,b_\ell\}\subseteq X$ of $\ell$ points, let $\Pi_C(B)=\{(c(b_1),\ldots,c(b_\ell)):c\in C\}$ be the set of all dichotomies that are realized by $C$ on $B$. We say that the set $B\subseteq X$ is {\em shattered} by $C$ if $C$ realizes all possible dichotomies over $B$, in which case we have $\left|\Pi_C(B)\right|=2^{|B|}$.

The VC dimension of the class $C$, denoted $\VC(C)$, is the cardinality of the largest set $B\subseteq X$ shattered by $C$. 
\end{definition}

\begin{theorem}[VC bound~\cite{}]
\label{thm:VCbound}
Let $C$ be a concept class over a domain $X$. For $\alpha,\beta<1/2$, there exists an $(\alpha,\beta,n)$-PAC learner for $C$, where $n=\frac{8\VC(C)\log(\frac{13}{\alpha})+4\log(\frac{2}{\beta})}{\alpha}$. 
\end{theorem}

\section{Proof of Theorem~\ref{thm:SdoesNotSuffice}}
\label{appendix:proof:SdoesNotSuffice}

The proof of Theorem~\ref{thm:SdoesNotSuffice} follows from algorithms \texttt{HypothesisLearner}, \texttt{AccuracyBoost} and claims~\ref{clm:HypothesisLearner},~\ref{clm:AccuracyBoost}, all described below.

In Algorithm \texttt{HypothesisLearner} we assume that the everlasting differentially private prediction interface $\AAA$ was fed with $n$ i.i.d.\ samples taken from some (unknown) distribution ${\cal D}$ and labeled by an unknown concept $c\in C$. Assumning the sequence of hypotheses $\{h_r\}_{r\geq 0}$ produced by $\AAA$ satisfies 
\begin{equation}   
\label{eq:GoodHypotheses}
\forall r~~ \mbox{error}_\DDD(c,h_r)\leq\alpha
\end{equation} 
we use it to construct -- with constant probability -- a hypothesis $h$ with error bounded by $O(\alpha)$.

\begin{algorithm}
\caption{\texttt{HypothesisLearner}} 
{\bf Parameters:} $0<\beta\leq 1/8$,
$R=|X|\log(|X|)\log(1/\beta)$

{\bf Input:} hypothesis sequence $\{h_r\}_{r\geq 0}$

\begin{enumerate}
\item for all $x\in X$ let $L_x=\emptyset$

\item for $r= 0,1,2,\ldots, R$
\begin{enumerate}[rightmargin=10pt,itemsep=1pt,topsep=0pt]

\item select $x$ uniformly at random from $X$ and let $L_{x}=L_{x}\cup \{h_r(x)\}$

\end{enumerate}

\item \label{step:checkFail} if $L_x=\emptyset$ for some $x\in X$ then fail, output an arbitrary hypothesis, and halt \\
{\color{gray} /* $\Pr[\exists x ~\mbox{such that}~L_x=\emptyset]\leq |X| (1-\frac{1}{|X|})^R\approx |X|e^{-R/|X|}=\beta$ */}

\item \label{step:chooseFromL} for all $x\in X$ let $r_x$ be sampled uniformly at random from $L_{x}$

\item construct the hypothesis $h$, where $h(x)=r_x$

\end{enumerate}
\end{algorithm}

\begin{claim} \label{clm:HypothesisLearner} 
If  executed on a hypothesis sequence satisfying Equation~\ref{eq:GoodHypotheses} then with probability at least $3/4$ Algorithm \texttt{\rm HypothesisLearner} outputs a hypothesis $h$ satisfying $\error_{\cal D}(c,h) \leq 8\alpha$. 
\end{claim}

\begin{proof}
Having ${\cal D}, c\in C$ fixed, and given a hypothesis $h$, we define $e_{h}(x)$ to be $1$ if $h(x)\not=c(x)$ and $0$ otherwise. 
Thus, we can write $\error_{\cal D}(c,h) = \mathbb{E}_{x\sim\DDD}[e_h(x)]$. 

Observe that when Algorithm \texttt{HypothesisLearner} does not fail, $r_x$ (and hence $h(x)$) is chosen with equal probability among $(h_1(x),h_2(x),\ldots,h_R(x))$ and hence $\mathbb{E}_\theta[e_h(x)]=\mathbb{E}_{i\in_R [R]}[e_{h_i}(x)]$ where $\theta$ denotes the randomness of \texttt{HypothesisLearner}. We get:
\begin{eqnarray*}
\mathbb{E}_{\theta}[\error_{\cal D}(c,h)] & = &
\mathbb{E}_{\theta}\mathbb{E}_{x\sim\DDD}[e_{h}(x)] =  \mathbb{E}_{x\sim\DDD} \mathbb{E}_{\theta} [e_{h}(x)]\\
&= & \mathbb{E}_{x\sim\DDD}\mathbb{E}_{i\in_R[R]}[e_{h_i}(x)] = \mathbb{E}_{i\in_R[R]}\mathbb{E}_{x\sim\DDD}[e_{h_i}(x)]\\
&\leq &\mathbb{E}_{i\sim\mathcal{R}} [\alpha] = \alpha.
\end{eqnarray*}
By Markov inequality, we have $\Pr_{\theta}[\error_{\cal D}(c,h)\geq 8\alpha]\leq 1/8$. The claim follows noting that Algorithm \texttt{HypothesisLearner} fails with probability at most $\beta\leq 1/8$.
\end{proof}

The second part of the transformation is Algorithm \texttt{AccuracyBoost} that applies Algorithm \texttt{HypothesisLearner} $O(\log (1/\beta))$ times to obtain with high probability a hypothesis with $O(\alpha)$ error.

\begin{algorithm}
\caption{\texttt{AccuracyBoost}} 
{\bf Parameters:} $\beta$, $R=104\ln\frac{1}{\beta}$

{\bf Input:} $R$ labeled samples with $n$ examples each $(S_1,\ldots,S_R)$ where $S_i\in(X\times\{0,1\})^n$

\begin{enumerate}
\item for $i= 1,2\ldots R$
\begin{enumerate}[rightmargin=10pt,itemsep=1pt,topsep=0pt]

\item \label{step:executeAAA} execute $\AAA(S_i)$ to obtain a hypothesis sequence $\{h_r^i\}_{r\geq 0}$

\item execute Algorithm  $\texttt{WeakHypothesisLearner}$ on $\{h_r^i\}_{r\geq 0}$ to obtain hypothesis $h^i$

\end{enumerate}

\item construct the hypothesis $\hat h$, where $\hat h(x)=\mbox{maj}(h^1(x),\ldots,h^R(x))$.

\end{enumerate}
\end{algorithm}

\begin{claim} \label{clm:AccuracyBoost}
    With probability $1-\beta$, Algorithm $\texttt{AccuracyBoost}$ output a $24\alpha$-good hypothesis over distribution $\DDD$.
\end{claim}

\begin{proof}
Define $B_i$ to be the event where the sequence of hypotheses $\{h_r^i\}_{r\geq 0}$ produced in Step~\ref{step:executeAAA} of \texttt{AccuracyBoost} does not satisfy Equation ~\ref{eq:GoodHypotheses}. We have,
$$\Pr[\error_{\cal D}(c,h_i)>8\alpha] \leq  \Pr[B] + (1-\Pr[B])\cdot \Pr[\error_{\cal D}(c,h) > 8\alpha] \leq \beta + 1/4 < 3/8.$$
Hence, by the Chernoff bound, when $R\geq 104\ln\frac{1}{\beta}$, we have at least $7R/8$ hypotheses are $8\alpha$-good over distribution $\DDD$. Consider the worst case, in which $R/8$ hypotheses always output wrong labels. To output a wrong label of $x$, we require at least $3R/8$ hypotheses to output wrong labels. Thus $h$ is $24\alpha$-good over distribution $\DDD$.
\end{proof}

\section{Tools from Prior Works} \label{appendix:tools}

\subsection{Algorithm \texttt{LabelBoost}~\citep{BeimelNS21}} \label{sec:labelBoost}

\begin{algorithm}
\caption{\bf \texttt{LabelBoost}~\citep{BeimelNS21}}\label{alg:LabelBoost} 
{\bf Parameters:} A concept class $C$.

{\bf Input:} A partially labeled database $S{\circ}T\in(X\times\{0,1,\bot\})^*$.
\begin{enumerate}[label=\gray{\%},topsep=-10pt,rightmargin=10pt,leftmargin=12pt]
\item \gray{
We assume that the first portion of the database (denoted $S$) contains labeled examples. The algorithm outputs a similar database where both $S$ and $T$ are (re)labeled. 
}
\end{enumerate}
\begin{enumerate}[rightmargin=10pt,itemsep=1pt,topsep=0pt]

\item Initialize $H=\emptyset$.

\item Let $P=\{p_1,\ldots,p_\ell\}$ be the set of all points $p\in X$ appearing at least once in $S{\circ}T$. Let $\Pi_C(P)=\{\left( c(p_1),\ldots,c(p_\ell) \right) :c\in C\}$ be the set of all dichotomies generated by $C$ on $P$.

\item For every $(z_1,\ldots,z_\ell)\in \Pi_C(P)$, add to $H$ an arbitrary concept $c\in C$ s.t.\ $c(p_i)=z_i$ for every $1\leq i\leq\ell$.

\item \label{step:Oneexpmech} Choose $h\in H$ using the exponential mechanism with privacy parameter $\epsilon{=}1$, solution set $H$, and the database $S$.

\item \label{step:Onerelabel} (Re)label $S{\circ}T$ using $h$, and denote the resulting database $(S{\circ}T)^h$, that is, if $S{\circ}T=(x_i,y_i)_{i=1}^t$ then $(S{\circ}T)^h=(x_i,y'_i)_{i=1}^t$ where $y'_i=h(x_i)$.

\item \label{step:OneAAA} Output $(S{\circ}T)^h$.

\end{enumerate}
\end{algorithm}

\begin{lem}[privacy of Algorithm \texttt{LabelBoost}~\citep{BeimelNS21}]\label{lemma:LabelBoostPrivacy}
Let $\AAA$ be an $(\epsilon,\delta)$-differentially private algorithm operating on partially labeled databases.
Construct an algorithm $\BBB$ that on input a partially labeled database $S{\circ}T\in(X\times\{0,1,\bot\})^*$ applies $\AAA$ on the outcome of $\texttt{LabelBoos}(S{\circ}T)$.
Then, $\BBB$ is $(\epsilon+3,4e\delta)$-differentially private.
\end{lem}

Consider an execution of \texttt{LabelBoost} on a database $S{\circ}T$, and assume that the examples in $S$ are labeled by some target concept $c\in C$.
Recall that for every possible labeling $\vec{z}$ of the elements in $S$ and in $T$, algorithm \texttt{LabelBoost} adds to $H$ a hypothesis from $C$ that agrees with $\vec{z}$.
In particular, $H$ contains a hypothesis that agrees with the target concept $c$ on $S$ (and on $T$). That is, $\exists f\in H$ s.t.\ $\error_S(f)=0$.
Hence, the exponential mechanism (on Step~\ref{step:Oneexpmech}) chooses (w.h.p.) a hypothesis $h\in H$ s.t.\ $\error_S(h)$ is small, provided that $|S|$ is roughly $\log|H|$, which is roughly $\VC(C)\cdot\log(|S|+|T|)$ by Sauer's lemma.
So, algorithm \texttt{LabelBoost} takes an input database where only a small portion of it is labeled, and returns a similar database in which the labeled portion grows exponentially.

\begin{lemma}[utility of Algorithm \texttt{LabelBoost}~\citep{BeimelNS21}]\label{lem:LabelBoostUtility}
Fix $\alpha$ and $\beta$, and let $S{\circ}T$ be s.t.\ $S$ is labeled by some target concept $c\in C$, and s.t.\ 
$$|T|\leq\frac{\beta}{e} \VC(C)\exp(\frac{\alpha |S|}{2\VC(C)})-|S|.$$
Consider the execution of \texttt{LabelBoost} on $S{\circ}T$, and let $h$ denote the hypothesis chosen on Step~\ref{step:Oneexpmech}.
With probability at least $(1-\beta)$ we have that $\error_{S}(h)\leq\alpha$.
\end{lemma}

\subsection{Algorithm \texttt{BetweenThresholds}~\citep{BunSU16}}

\begin{algorithm}
\caption{\bf \texttt{BetweenThresholds}~\citep{BunSU16}}\label{alg:BetweenThresholds}
{\bf Input:} Database $S\in X^n$.\\
{\bf Parameters:} $\eps,\lT,\uT \in (0,1)$ and $n, k \in \N$.
\begin{enumerate}[rightmargin=10pt,itemsep=1pt,topsep=1pt]

\item Sample $\mu \sim \Lap(2/\eps n)$ and initialize noisy thresholds $\hlT = \lT + \mu$ and $\huT = \uT - \mu$.

\item For $j = 1, 2, \cdots, k$:
\begin{enumerate}[rightmargin=10pt,itemsep=1pt,topsep=0pt]
	\item Receive query $q_j : X^n \to [0,1]$.
	\item Set $c_j = q_j(S) + \nu_j$ where $\nu_j \sim \Lap(6/\eps n)$.
	\item If $c_j < \hlT$, output $\leftsymb$ and continue.
	\item If $c_j > \huT$, output $\rightsymb$ and continue.
	\item If $c_j \in [\hlT,\huT]$, output $\haltsymb$ and halt.
\end{enumerate}

\end{enumerate}
\end{algorithm}

\begin{lem}[Privacy for \texttt{BetweenThresholds}~\citep{BunSU16}] \label{lem:bt-privacy}
Let $\eps,\delta \in (0,1)$ and $n \in \N$. 
Then algorithm \texttt{BetweenThresholds} is $(\eps, \delta)$-differentially private for any adaptively-chosen sequence of queries as long as the gap between the thresholds $\lT, \uT$ satisfies
\[\uT - \lT \ge \frac{12}{\eps n}\left( \log (10/\eps) + \log(1/\delta) + 1\right).\]
\end{lem}

\begin{lem}[Accuracy for \texttt{BetweenThresholds}~\citep{BunSU16}] \label{lem:bt-accuracy}
Let $\alpha, \beta,\eps,\lT,\uT \in (0,1)$ and $n,k \in \N$ satisfy \[n \geq \frac{8}{\alpha \eps}\left(\log(k+1) + \log(1/\beta)\right).\] Then, for any input $x \in {X}^n$ and any adaptively-chosen sequence of queries $q_1, q_2, \cdots, q_k$, the answers $a_1, a_2, \cdots a_{\leq k}$ produced by \texttt{BetweenThresholds} on input $x$ satisfy the following with probability at least $1-\beta$. For any $j \in [k]$ such that $a_j$ is returned before \texttt{BetweenThresholds} halts,
\begin{itemize}
\item $a_j = \leftsymb \implies q_j(x) \le \lT + \alpha$,
\item $a_j = \rightsymb \implies q_j(x) \ge \uT - \alpha$, and
\item $a_j = \haltsymb \implies \lT - \alpha \le q_j(x) \le \uT + \alpha$.
\end{itemize}
\end{lem}

\begin{obs}\label{bt-obs}
Using standard composition theorems for differential privacy (see, e.g.,~\cite{DRV10}), we can assume that algorithm \texttt{BetweenThresholds} takes another parameter $c$, and halts after $c$ times of outputting $\top$. In this case, the algorithm satisfies $(\eps',2c\delta)$-differential privacy, for $\eps'=\sqrt{2c\ln(\frac{1}{c\delta})}\eps+c\eps(e^\eps-1)$.
\end{obs}

\section{Some Technical Facts}\label{appendix:tech-facts}
We refer to the execution of steps~\ref{step:beginloop}-\ref{step:endloop} of algorithm \texttt{GenericBBL} as a {\em phase} of the algorithm, indexed by $i=1,2,3,\dots$. 

The original \texttt{BetweenThresholds} needs to halt when it outputs $\top$. In \texttt{GenericBBL}, we tolerance it to halt at most $c_i$ times in the phase $i$. We prove \texttt{BetweenThresholds} in \texttt{GenericBBL} is $(1,\delta)$-differentially private.
\begin{claim}
For $\delta<1$, Mechanism \texttt{BetweenThresholds} used in step~\ref{GenericBBL_unlabelR} in the $i$-th iteration, is $(1,\delta)$-differentially private.
\end{claim}
\begin{proof} Let $\eps'_i,\delta'_i$ be as in Step~\ref{GenericBBL_unlabelR}.
Since $e^{\eps'_i}-1<2\eps'_i$ for $0<\eps'_i<1$, we have
$$
\sqrt{2c_i\ln(\frac{1}{c_i\delta'_i})}\cdot \eps'_i+c_i\eps'_i(e^{\eps'_i}-1) \leq \sqrt{2c_i\ln(\frac{2}{\delta})}\cdot\eps'_i+2c_i{\eps'_i}^2
= \frac{\sqrt{2}}{3} +\frac{2}{9\ln(\frac{2}{\delta})}
\leq 1.
$$
The proof is concluded by using  observation~\ref{bt-obs}.
\end{proof}

In Claim~\ref{clm:prob-fi-alpha-good}-~\ref{betweenthresholds-fail-ctimes}, we prove that with high probability, \texttt{BetweenThresholds} in step~\ref{GenericBBL_main_loop} halts within $64\alpha_i$ times. We prove it by 4 steps: \\
1.prove that with high probability, most hypothesis in step~\ref{step:generatehypothesis} have high accuracy (Claim~\ref{clm:prob-fi-alpha-good}). \\
2.prove that if most hypothesis in step~\ref{step:generatehypothesis} have high accuracy, then with high probability, the queries in \texttt{BetweenThresholds} are closed to 0 or 1 (Claim~\ref{clm:prob-qx-good}).\\
3.prove that if the queries in \texttt{BetweenThresholds} are closed to 0 or 1, then \texttt{BetweenThresholds} in step~\ref{GenericBBL_main_loop} will outputs $L$ or $R$ with high probability(Claim~\ref{claim:betweenthresholds-singl-query-fail}).\\
4.prove that if \texttt{BetweenThresholds} outputs $L$ or $R$, then every single phase fails with low probability(Claim~\ref{betweenthresholds-fail-ctimes}).

\begin{claim}\label{clm:prob-fi-alpha-good}
    If $\beta_i\leq 1/32$ and $T_i\geq 96\ln{\frac{1}{\alpha_i}}$, then with probability $1-\alpha_i$, $\frac{15T_i}{16}$ hypotheses in step~\ref{step:generatehypothesis} are $\alpha_i$-good with respect to $g_i$, where $g_i$ is the concept of $S_i$.
\end{claim}
\begin{proof}
    By the VC bound (Theorem~\ref{thm:VCbound}), for each $t\in[T_i]$, we have
    $$
    \Pr[\error_\DDD(f_t,g_i)\leq \alpha_i]\geq 1-\beta_i.
    $$
    By Chernoff bound, if $T_i\geq \frac{16+256\beta_i}{(1-16\beta_i)^2}\ln\frac{1}{\alpha_i}$, then with probability $1-\alpha_i$, we have $\frac{15T_i}{16}$ hypotheses have $\error_\DDD(f_t,g_i)\leq\alpha_i$. When $\beta_i\leq 1/32$, it is sufficient to set $T_i\geq 96\ln\frac{1}{\alpha_i}$.
\end{proof}

\begin{claim}\label{clm:prob-qx-good}
    If $\alpha_i\leq 1/16$ and $\frac{15T_i}{16}$ hypotheses in step~\ref{step:generatehypothesis} are $\alpha_i$-good with respect to $g_i$, where $g_i$ is the concept of $S_i$, then $\Pr_{x\sim\DDD}[|q(x)-\frac{1}{2}|\leq \frac{3}{8}]\leq 15\alpha_i$.
\end{claim}
\begin{proof}
    W.l.o.g. assume $g_i(x)=1$, where $g_i$ is the concept of $S_i$, so it is sufficient to prove $\Pr_{x\sim\DDD}[q(x)\leq \frac{7}{8}]\leq 8\alpha_i$. Consider the worst case that $\frac{T_i}{16}$ ''bad" hypotheses output 0. In that case, $q(x)\leq \frac{7}{8}$ when $\frac{T_i}{16}$ of $\alpha_i$-good hypotheses output 0. So that with probability $15\alpha_i$, we have $q(x)\leq\frac{7}{8}$.(see Figure~\ref{fig:queries-in-BT})

    \begin{figure}[h]
    \centering
    \includegraphics[width=0.9\textwidth]{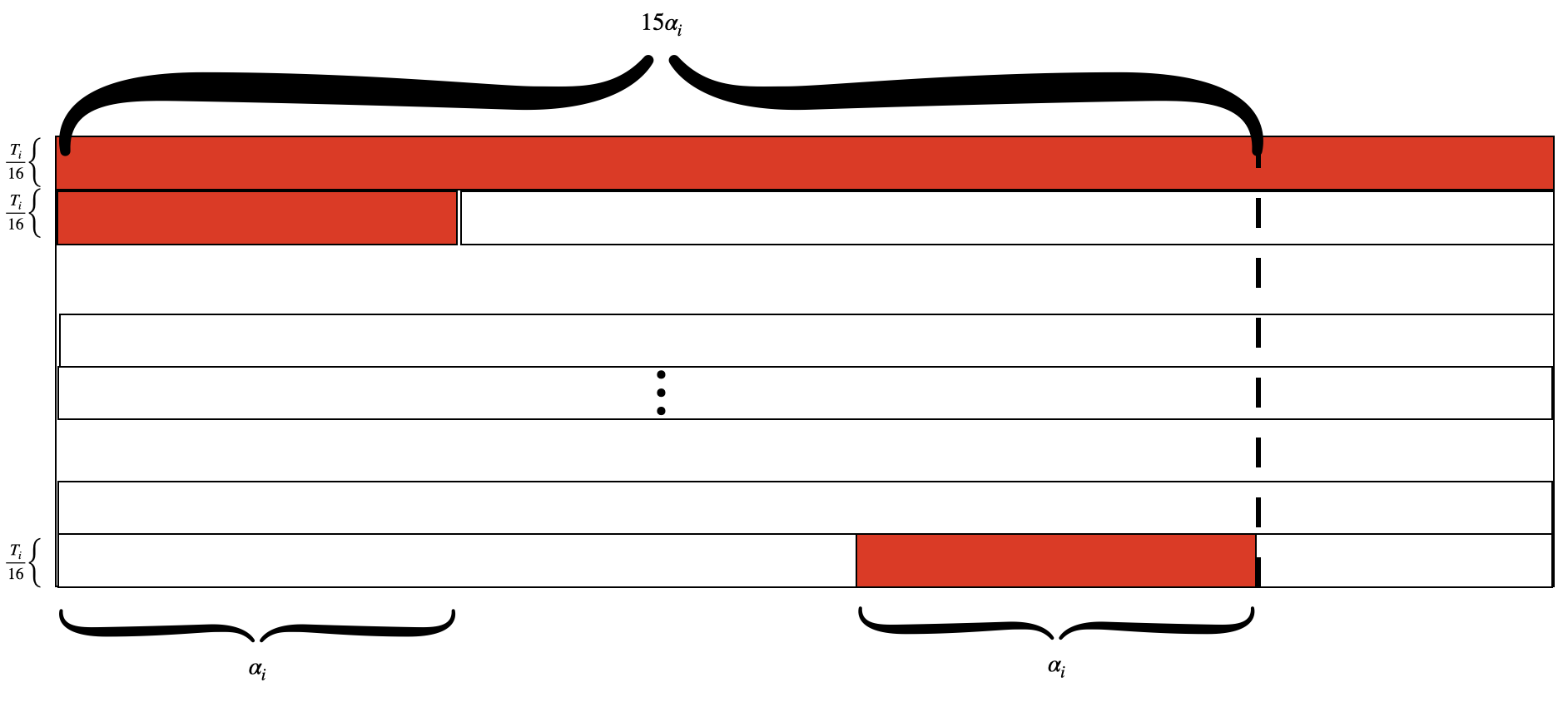}
    \caption{The horizontal represents the input point. The vertical represents the hypothesis. The red parts represent the incorrect prediction. We let $\frac{T_i}{16}$ hypothesis predict all labels as 0. To let $q(x)\leq\frac{7}{8}$, there must exist $\frac{T_i}{16}$ hypothesis output 0. In the worst case, at most $15\alpha_i$ of points are labeled as 0.}\label{fig:queries-in-BT}

    \end{figure}
\end{proof}

\begin{claim}\label{claim:betweenthresholds-singl-query-fail} Let $t_u<1/2+1/8$ and $t_\ell>1/2-1/8$. 
    For a query $q$ such that $q(S)>7/8$ (similarly, for $q(S)<1/8$), Algorithm \texttt{BetweenThresholds} outputs $R$ (similarly, $L$) with probability at least $1-\mbox{exp}\left(-\frac{T_i}{144\sqrt{c_i\ln(\frac{2}{\delta})}}\right)$. 
\end{claim}
\begin{proof}
Wlog assume $q(S)>7/8$, it is sufficient to show

\begin{eqnarray*}
\Pr[\mbox{\texttt{BetweenThreshold} outputs} R] & 
= & \Pr[q(S)+\mbox{Lap}(6/\eps'T_i)>t_u+\mbox{Lap}(2/\eps'T_i)]\\
& > & \Pr[\mbox{Lap}(6/\eps'T_i)>-1/8]\cdot\Pr[\mbox{Lap}(2/\eps'T_i)<1/8]\\
& = & \left(1-\frac{1}{2}\mbox{exp}\left(-\frac{T_i}{144\sqrt{c_i\ln(\frac{2}{\delta})}}\right)\right)\cdot\left(1-\frac{1}{2}\mbox{exp}\left(-\frac{T_i}{48\sqrt{c_i\ln(\frac{2}{\delta})}}\right)\right)\\
& > & 1-\mbox{exp}\left(-\frac{T_i}{144\sqrt{c_i\ln(\frac{2}{\delta})}}\right).
\end{eqnarray*}
\end{proof}

\begin{claim}\label{betweenthresholds-fail-ctimes}
For any phase $i$, \texttt{BetweenThresholds} outputs $\top$ at most $64\alpha_iR_i$ times with probability at most $\beta_i$.
\end{claim}
\begin{proof}
For a single query, if $t_u<1/2+1/8$  and $q(S)>7/8$ (similarly, $t_\ell>1/2-1/8$ and $q(S)<1/8$), by Claim~\ref{claim:betweenthresholds-singl-query-fail}, \texttt{BetweenThresholds} outputs $\top$ with probability at most $\mbox{exp}\left(-\frac{T_i}{144\sqrt{c_i\ln(\frac{2}{\delta})}}\right)=\mbox{exp}\left(-\frac{T_i}{144\sqrt{64\alpha_iR_i\ln(\frac{2}{\delta})}}\right)<\alpha_i$. Combine Claim~\ref{clm:prob-fi-alpha-good} and \ref{clm:prob-qx-good}, \texttt{BetweenThresholds} outputs $\top$ with probability at most $32\alpha_i$. By the Chernoff bound and $R_i\geq\frac{3\ln(\frac{1}{\beta_i})}{\alpha_i}$, \texttt{BetweenThresholds} outputs $\top$ more than $64\alpha_iR_i$ times with probability at most $\beta_i$.
\end{proof}

In step~\ref{GenericBBL_LabelBoost}, \texttt{GenericBBL} takes a random subset of size $\lambda_{i+1}T_{t+1}$ from $\hat{D}'_i$. We show that the size of $\hat{D}'_i$ is at least $\lambda_{i+1}T_{t+1}$.
\begin{claim}\label{Si-size}
When $\varepsilon\leq 1$, for any $i\geq 1$, we always have
$|\hat{D}'_i|\geq \lambda_{i+1}T_{i+1}$.
\end{claim}
\begin{proof}
 Let $m=3+\mbox{exp}(\varepsilon+4)<200$. By the step~\ref{GenericBBL_unlabelR}, step~\ref{GenericBBL_UnlabeledDatabases} and step~\ref{GenericBBL_LabelBoost}, $|\hat{D}_j|=\frac{\varepsilon|D_j|}{m}=\frac{25600|S_j|}{m}\geq 128|S_j|=128\lambda_jT_j$. Then it is sufficient to verify $128\lambda_jT_j\geq \lambda_{j+1}T_{j+1}$

We can verify that 
$$
4\lambda_j=4\cdot\frac{8\VC(C)\log(\frac{13}{\alpha_i})+4\log(\frac{2}{\beta_i})}{\alpha_i}\\
=4\cdot\frac{8\VC(C)(\log(\frac{13}{\alpha_{j+1}})-1)+4(\log(\frac{2}{\beta_{j+1}})-1)}{2\alpha_{j+1}}\geq\lambda_{j+1}
$$
and
$$
32T_j=\frac{32\tau\cdot\lambda_i\cdot\log(\frac{1}{\delta})\cdot\log^2(\frac{\lambda_i}{\eps\alpha_i\beta_i\delta})}{\alpha_i\eps}
\geq\frac{32\tau\cdot\lambda_i\cdot\log(\frac{1}{\delta})\cdot\log^2(\frac{\lambda_{i+1}}{16\eps\alpha_{i+1}\beta_{i+1}\delta})}{8\alpha_{i+1}\eps}\geq\lambda_{j+1}T_{j+1}.
$$
The last inequalitu holds because $\lambda_j\geq4$ and $\alpha_j,\beta_j\leq1/2$.
\end{proof}

To apply the privacy and accuracy of $LabelBoost$ and $BetweenThresholds$, the sizes of the databases need to satisfy the inequalities in lemma~\ref{lem:LabelBoostUtility},~\ref{lem:bt-privacy} and ~\ref{lem:bt-accuracy}. We verify that in each phase, the sizes of the databases always satisfy the requirement.
\begin{claim}\label{claim:verify-betweenthreshold-accuracy}
Let $\alpha,\beta,\delta<1/16$, $\varepsilon\leq 1$, and $\VC(C)\geq1$. Then for any $i\geq 1$, we have $$T_i\geq \frac{8}{\alpha_i \eps'}\left(\log(|D_i|+1) + \log(1/\beta_i)\right).$$
\end{claim}
\begin{proof}
By claim~\ref{Si-size} and step~\ref{GenericBBL_unlabelR}, $|D_i|=\frac{25600|S_i|}{\eps}=\frac{25600\lambda_iT_i}{\eps}$. 
Since
\begin{equation*}
\begin{split}
\frac{8}{\alpha_i \eps'}\left(\log(|D_i|+1) + \log(1/\beta_i)\right) &=\frac{24\sqrt{64\alpha_i|D_i|\ln(\frac{2}{\delta})}}{\sqrt{2}\alpha_i}\cdot\left(\log(|D_i|+1) + \log(1/\beta_i)\right) \\& =
O\left(\sqrt{\frac{\lambda_iT_i\log(\frac{1}{\delta})}{\alpha_i\eps}}\left(\log(\frac{\lambda_iT_i}{\eps\beta_i})\right)\right)\\ & =O\left(\sqrt{\frac{\lambda_iT_i\log(\frac{1}{\delta})}{\alpha_i\eps}}\cdot\log\left(\frac{\lambda_i\log(\frac{1}{\delta})}{\alpha_i\beta_i\eps}\right)\right),
\end{split}
\end{equation*}
and $T_i=\frac{\tau\cdot\lambda_i\cdot\log(\frac{1}{\delta})\cdot\log^2(\frac{\lambda_i}{\eps\alpha_i\beta_i\delta})}{\alpha_i\eps}$, where $\tau\geq1.1*10^{10}$, the inequality always holds.
\end{proof}

\begin{claim}\label{clm:labelboost-datasize}
When $\varepsilon\leq 1$, for any $i\geq 1$, we have $|\hat{D}_i|\leq\frac{\beta_i}{e}\VC(C)\mbox{exp}\left(\frac{\alpha_i|\hat{S}_i|}{2\VC(C)}\right)-|\hat{S}_i|$.
\end{claim}

\begin{proof}
By claim~\ref{Si-size}, step~\ref{GenericBBL_unlabelR} and step~\ref{GenericBBL_LabelBoost}, $$|\hat{D}_i|=\frac{\eps|D_i|}{m}=O\left(\lambda_iT_i\right)=O\left(\VC(C)\log^2(\VC(C))\cdot \mbox{poly}\left(\frac{1}{\alpha_i},\log(\frac{1}{\beta_i}),\frac{1}{\eps},\log(\frac{1}{\delta})\right)\right)$$ 
and 
\begin{equation} 
\begin{split}
|\hat{S}_i|& =\frac{\eps|S_i|}{m}\\ &=O\left(\eps\lambda_iT_i\right)=O\left(\lambda_iT_i\right)\\ &=O\left(\VC(C)\log^2(\VC(C))\cdot \mbox{poly}\left(\frac{1}{\alpha_i},\log(\frac{1}{\beta_i}),\frac{1}{\eps},\log(\frac{1}{\delta})\right)\right).
\end{split}
\end{equation}
Note that $$\frac{\beta_i}{e}\VC(C)\mbox{exp}\left(\frac{\alpha_i|\hat{S}_i|}{2\VC(C)}\right)=\Omega\left(\VC^2(C)\cdot \mbox{exp}\left(\mbox{poly}\left(\frac{1}{\alpha_i},\log(\frac{1}{\beta_i}),\frac{1}{\eps},\log(\frac{1}{\delta})\right)\right)\right),$$ for $T_i=\frac{\tau\cdot\lambda_i\cdot\log(\frac{1}{\delta})\cdot\log^2(\frac{\lambda_i}{\eps\alpha_i\beta_i\delta})}{\alpha_i\eps}$, the inequality holds when $\tau\geq 1$.
\end{proof}

\begin{claim}
For every $i\geq 1$, we have \[\uT - \lT \ge \frac{12}{\eps'_i T_i}\left( \log (10/\eps'_i) + \log(1/\delta'_i) + 1\right).\]
\end{claim}

\begin{proof}
By step~\ref{GenericBBL_unlabelR}, $t_u-t_{\ell}=2\alpha_i$. Then we have 
$$
\begin{array}{rl}
     \frac{6}{\alpha_i\eps'_i T_i}\left( \log (10/\eps'_i) + \log(1/\delta'_i) + 1\right)
     & =\frac{6\sqrt{64\alpha_iR_i\ln(\frac{2}{\delta})}}{\alpha_iT_i}\left( \log (10/\eps'_i) + \log(1/\delta'_i) + 1\right) \\
     & =6\sqrt{\frac{1638400\ln(\frac{2}{\delta})\lambda_i}{\alpha_iT_i}}\left( \log (10/\eps'_i) + \log(1/\delta'_i) + 1\right)\\
     & = 6\sqrt{\frac{1638400\ln(\frac{2}{\delta})}{\tau\log(\frac{1}{\delta})\log^2(\frac{\lambda_i}{\eps\alpha_i\beta_i\delta})}}\left( \log (10/\eps'_i) + \log(1/\delta'_i) + 1\right)\\
     & = O(1),
\end{array}
$$
the inequality holds when $\tau>10^{10}$.
\end{proof}

\section{Accuracy of Algorithm \texttt{GenericBBL} -- proof of Theorem~\ref{thm:accuracy}}
\label{appendix:accuracy}

We refer to the execution of steps~\ref{step:beginloop}-\ref{step:endloop} of algorithm \texttt{GenericBBL} as a {\em phase} of the algorithm, indexed by $i=1,2,3,\dots$. 

We give some technical facts in Appendix~\ref{appendix:tech-facts}. In Claim~\ref{clm:one-phase-accuracy}, we show that in each phase, samples are labeled with high accuracy. In Claim~\ref{clm:fail-all-phase}, we prove that algorithm \texttt{GenericBBL} fails with low probability. In Claim~\ref{clm: hypothesis-accuracy}, we prove that algorithm \texttt{GenericBBL} predict the labels with high accuracy.

\begin{claim}\label{clm:one-phase-accuracy}
When Algorithm \texttt{GenericBBL} does not fail on phases $1$ to $i$, then for phase $i+1$ we have
$$
\Pr\left[\exists g_{i+1}\in C \text{ s.t.\ } \error_{S_{i+1}}(g_{i+1})=0 \text{ and } \error_{\DDD}(g_{i+1},c)\leq\sum_{j=1}^{i+1}\alpha_j \right] \geq 1-{2\sum_{j=0}^{i+1}\beta_j}.
$$
\end{claim}

\begin{proof}
The proof is by induction on $i$. The base case for $i=1$ is trivial, with $g_1=c$. Assume the claim holds for all $j\leq i$. By the properties of \texttt{LabelBoost} (Lemma~\ref{lem:LabelBoostUtility}) and Claim~\ref{clm:labelboost-datasize}, with probability at least $1-\beta_{i+1}$ we have that $S_{i+1}$ is labeled by a hypothesis $g_{i+1}\in C$ s.t.\ $\error_{S_i}(g_i,g_{i+1})\leq\alpha_{i+1}$. 
Observe that the points in $S_i$ (without their labels) are chosen i.i.d.\ from $\DDD$, and hence, By Theorem~\ref{thm:VCbound} (VC bounds) and $|S_i|\geq 128\lambda_i\geq \lambda_{i+1}$, with probability at least $1-\beta_{i+1}$ we have that $\error_{\DDD}(g_i,g_{i+1})\leq\alpha_{i+1}$. 
Hence, with probability $1-2\beta_{i+1}$, we have $\error_{\DDD}(g_i,g_{i+1})\leq\alpha_{i+1}$. Finally, by the triangle inequality, $\error_{\DDD}(g_{i+1},c)\leq\sum_{j=1}^{i+1}\alpha_j$, except with probability $2\sum_{j=1}^{i+1}\beta_j$
\end{proof}

Define the following good event.

\begin{center}
\noindent\fboxother{
\parbox{.9\columnwidth}{
{\bf Event ${\boldsymbol{E_1}}$: } 
Algorithm \texttt{GenericBBL} never fails on the execution of 
\texttt{BetweenThresholds} in step~\ref{step:betweenThresholdFail}.
}}
\end{center}

\begin{claim}\label{clm:fail-all-phase}
Event $E_1$ occurs with probability at least $1-\beta$.
\end{claim}

\begin{proof}
Using to union bound and Claim~\ref{betweenthresholds-fail-ctimes}, 
$$
\Pr[\mbox{Event $E_1$ occurs}]\geq1-\beta.
$$
\end{proof}

Combining claims~\ref{clm:one-phase-accuracy} and~\ref{clm:fail-all-phase}, we get:
\begin{claim}\label{claim:LabelBoostStreach}
Let $\DDD$ be an underlying distribution and let $c\in C$ be a target concept. Then
$$
\Pr[\forall i\; \exists g_i\in C \text{ s.t.\ } \error_{S_i}(g_i)=0 \text{ and } \error_{\DDD}(g_i,c)\leq\alpha ] \geq 1-3\beta.
$$
\end{claim}

\paragraph{Notations.} Consider the $i$th phase of Algorithm \texttt{GenericBBL}, and focus on the $j$-th iteration of Step~\ref{step:mainloop}. Fix all of the randomness in \texttt{BetweenThresholds}. Now observe that the output on step~\ref{GenericBBL_label_point} is a deterministic function of the input $x_{i,j}$. This defines a hypothesis which we denote as $h_{i,j}$.

\begin{figure}[h]
    \centering
    \includegraphics[width=0.9\textwidth]{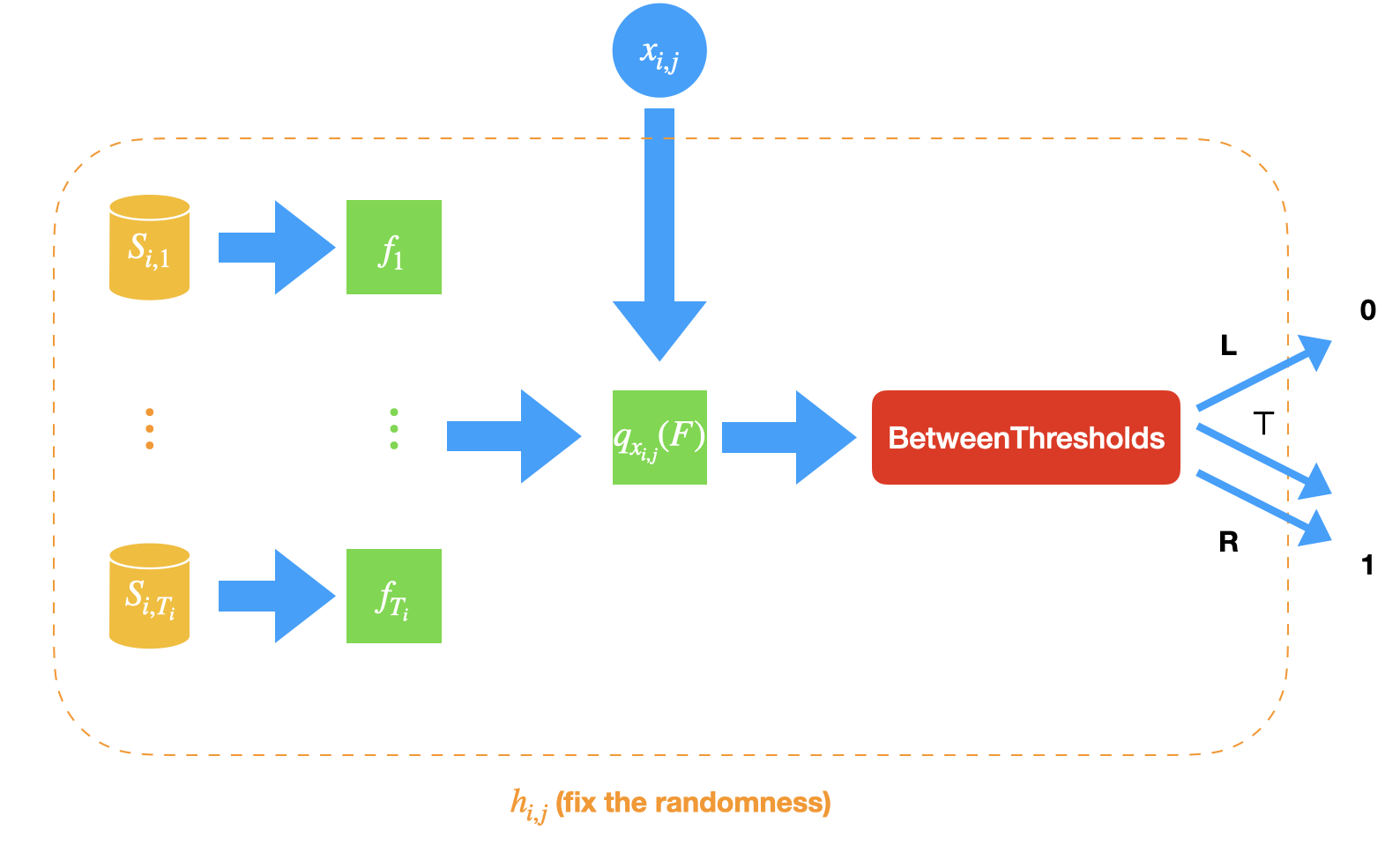}
    \caption{Hypothesis $h_{i,j}$}

\end{figure}

\begin{claim}\label{clm: hypothesis-accuracy}
For $\beta <1/16$, with probability at least $1-4\beta$, all of the hypotheses defined above are $6\alpha$-good w.r.t.\ $\DDD$ and $c$.
\end{claim}

\begin{proof}
In the phase $i$, by Claim~\ref{claim:LabelBoostStreach}, with probability at least $1-3\beta$ we have that $S_i$ is labeled by a hypothesis $g_i\in C$ satisfying $\error_{\DDD}(g_i,c)\leq\alpha$. We continue with the analysis assuming that this is the case. 

On step~\ref{step:beginloop} of the $i$th phase we divide $S_i$ into $T_i$ subsamples of size $\lambda_i$ each, identify a consistent hypothesis $f_t\in C$ for every subsample $S_{i,t}$, and denote $F_i=\left(f_1,\ldots,f_T\right)$. By Theorem~\ref{thm:VCbound} (VC bounds), every hypothesis in $F_i$ satisfies $\error_{\DDD}(f_t,g_i)\leq\alpha$ with probability $3/4$, in which case, by the triangle inequality we have that $\error_{\DDD}(f_t,c)\leq2\alpha$.

Set $T_i\geq \frac{512(1-4\beta_i)\ln(\frac{1}{\beta_i})}{(1-64\beta_i)^2}$, using Chernoff bound, it holds that for at least $15T_i/16$ of the hypotheses in $F_i$ have error $\error_{\DDD}(f_t,g_i)\leq2\alpha$ with probability at least $1-\beta_i$. These hypotheses have $\error_{\DDD}(f_t,c)\leq3\alpha$.

Let $m:X\rightarrow\{0,1\}$ defined as $m(x)=\mbox{maj}_{f_t\in F_i}(f_t(x))$. For $m$ to err on a point $x$ (w.r.t.\ the target concept $c$), it must be that at least $7/16$-fraction of the $3\alpha$-good hypotheses in $\hat{F}_i$ err on $x$. Consider the worst case in Figure~\ref{fig:hypothesis-accuracy} , we have $\error_{\DDD}(m,c)\leq6\alpha$

\begin{figure}[h]
    \centering
    \includegraphics[width=0.9\textwidth]{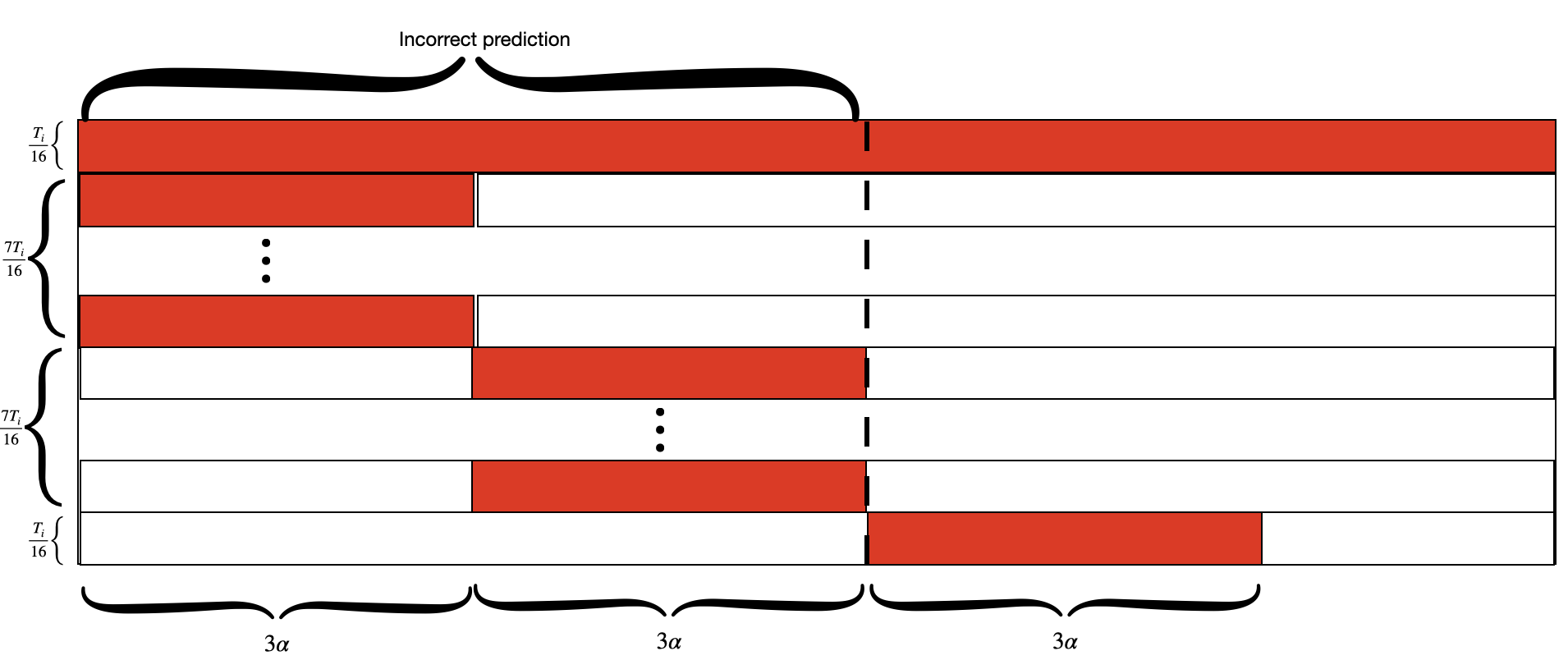}
    \caption{The horizontal represents the input point. The vertical represents the hypothesis. The red parts represent the incorrect prediction. We let $\frac{T_i}{16}$ hypothesis predict all labels incorrectly. To output an incorrect label, there must exist $\frac{7T_i}{16}$ hypothesis output the incorrect label. In the worst case, at most $6\alpha$ of points are incorrectly classified.}\label{fig:hypothesis-accuracy}

\end{figure}

By Lemma~\ref{lem:bt-accuracy} and Claim~\ref{claim:verify-betweenthreshold-accuracy}, with probability at least $1-\beta_i$, all of the hypotheses defined during the $i$th iteration satisfy this condition, and are hence $6\alpha$-good w.r.t.\ $c$ and $\DDD$. By the union bound, with probability $1-4\beta$, all the hypotheses are $6\alpha$-good.
\end{proof}

\subsection{Privacy analysis -- proof of Claim~\ref{clm:GenericBBLprivacy}}
\label{appendix:proofGenericBBLprivacy}

Fix $t\in\mathbb{N}$ and the adversary $\cal B$. We need to show that $\mbox{View}^0_{{\cal B}, t}$ and $\mbox{View}^1_{{\cal B}, t}$ (defined in Figure~\ref{fig:AdversarialExperiment}) are $(\varepsilon,\delta)-indistinguishable$. 
We will consider separately the case where the executions differ in the training phase (Claim~\ref{clm:privacy-S}) and the case where the difference occurs during the prediction phase (Claim~\ref{clm:privacy-D}).

\paragraph{Privacy of the initial training set $S$.}  Let $S^0,S^1 \in (X\times \{0,1\})^n$ be neighboring datasets of labeled examples and let $\mbox{View}^0_{{\cal B}, t}$ and $\mbox{View}^1_{{\cal B}, t}$ be as in Figure~\ref{fig:AdversarialExperiment} where $\left((x_1^0,y_1^0),\dots,(x_n^0,y_n^0)\right)=S^0$ and $\left((x_1^1,y_1^1),\dots,(x_n^1,y_n^1)\right)=S^1$.

\begin{claim}\label{clm:privacy-S}
For all adversaries $\cal B$, for all $t > 0$, and for any two neighbouring database $S^0$ and $S^1$ selected by $\cal B$,
$\mbox{View}^0_{{\cal B}, t}$ and $\mbox{View}^1_{{\cal B}, t}$ are $(\eps,\delta)$-indistinguishable.
\end{claim}

\begin{proof}

\begin{figure}[h]
    \centering
    \includegraphics[width=0.9\textwidth]{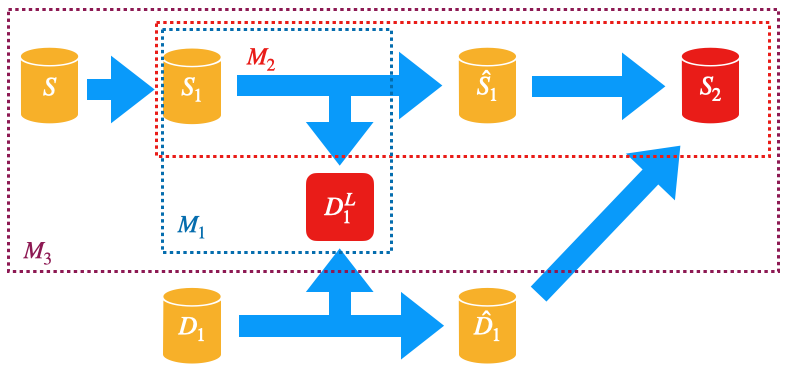}
    \caption{\label{fig:privacy_S} Privacy of the labeled sample $S$}
\end{figure}

Let $R'_1=\min(t,R_1)$. Note that $\mbox{View}^b_{{\cal B}, R'_1}$ is a prefix of $\mbox{View}^b_{{\cal B}, t}$ which includes the labels Algorithm \texttt{GenericBBL} produces in Step~\ref{GenericBBL_label_point} for the $R'_1$ first unlabeled points selected by $\cal B$. Let $S^b_2$ be the result of the first application of algorithm \texttt{LabelBoost} in Step~\ref{GenericBBL_LabelBoost} of \texttt{GenericBBL} (if $t< R_1$ we set $S^b_2$ as $\bot$). The creation of these random variables is depicted in Figure~\ref{fig:privacy_S}, where $D_1^L$ denotes the labels Algorithm~\texttt{GenericBBL} produces for the unlabeled points $D_1$.

Observe that $\mbox{View}^b_{{\cal B}, t}$ results from a post-processing (jointly by the adversary $\cal B$ and Algorithm~\texttt{GenericBBL}) of the random variable $\left(\mbox{View}^b_{{\cal B}, R'_1}, S^b_2\right)$, and hence it suffices to show that $\left(\mbox{View}^0_{{\cal B}, R'_1}, S^0_2\right)$ and $\left(\mbox{View}^1_{{\cal B}, R'_1}, S^1_2\right)$ are $(\varepsilon,\delta)$-indistinguishable. 

We follow the processes creating $\mbox{View}^b_{{\cal B}, t}$ and $S^b_2$ in Figure~\ref{fig:privacy_S}: 
(i) The mechanism $M_1$ corresponds to the loop in Step~\ref{GenericBBL_main_loop} of \texttt{GenericBBL} where labels are produced for the adversarially chosen points $D^b_1$.
By application of Lemma~\ref{lem:bt-privacy}, $M_1$ is $(1,\delta)$-differentially private.
(ii) The mechanism $M_2$, corresponds to the subsampling of $\hat{S}^b_1$ from $S^b_1$ and the application of procedure \texttt{LabelBoost} on the subsample in Step~\ref{GenericBBL_LabelBoost} of \texttt{GenericBBL} resulting in $S^b_2$. 
By application of Claim~\ref{claim:sub-sampling} and Lemma~\ref{lemma:LabelBoostPrivacy}, $M_2$ is $(\varepsilon,0)$-differentially private.
Thus $(M_1,M_2)$ is $(\varepsilon+1,\delta)$-differentially private. 
(iii) The mechanism $M_3$ with input of $S^b$ and output $\left(D_1^{b,L},S^b_2\right)=\left(\mbox{View}^b_{{\cal B}, R'_1},S^b_2\right)$ applies $(M_1,M_2)$ on the sub-sample $S^b_1$ obtained from $S^b$ in Step~\ref{GenericBBL_SubsampleS1} of \texttt{GenericBBL}. By application of Claim~\ref{claim:sub-sampling} $M_3$ is $(\varepsilon,\frac{4\varepsilon\delta}{3+\mbox{exp}(\varepsilon+1)})$-differentially private. Since $\frac{4\varepsilon\delta}{3+\mbox{exp}(\varepsilon+1)}\leq \delta$ for any $\varepsilon$, hence $\left(\mbox{View}^0_{{\cal B}, R'_1}, S^0_2\right)$ and $\left(\mbox{View}^1_{{\cal B}, R'_1}, S^1_2\right)$ are $(\varepsilon,\delta)$-indistinguishable
\end{proof}

\paragraph{Privacy of the unlabeled points $D$.}

Let $D^0,D^1 \in X^t$ be neighboring datasets of unlabeled examples and let $\mbox{View}^0_{{\cal B}, t}$ and $\mbox{View}^1_{{\cal B}, t}$ be as in Figure~\ref{fig:AdversarialExperiment} where $\left(x_1^0,\dots,x_t^0\right)=D^0$ and $\left(x_1^1,\dots,x_t^1\right)=D^1$.

\begin{claim}\label{clm:privacy-D}
For all adversaries $\cal B$, for all $t > 0$, and for any two neighbouring databases $D^0$ and $D^1$ selected by $\cal B$,
$\mbox{View}^0_{{\cal B}, t}$ and $\mbox{View}^1_{{\cal B}, t}$ are $(\eps,\delta)$-indistinguishable.
\end{claim}

\begin{proof}
\begin{figure}[h]
    \centering
    \includegraphics[width=0.9\textwidth]{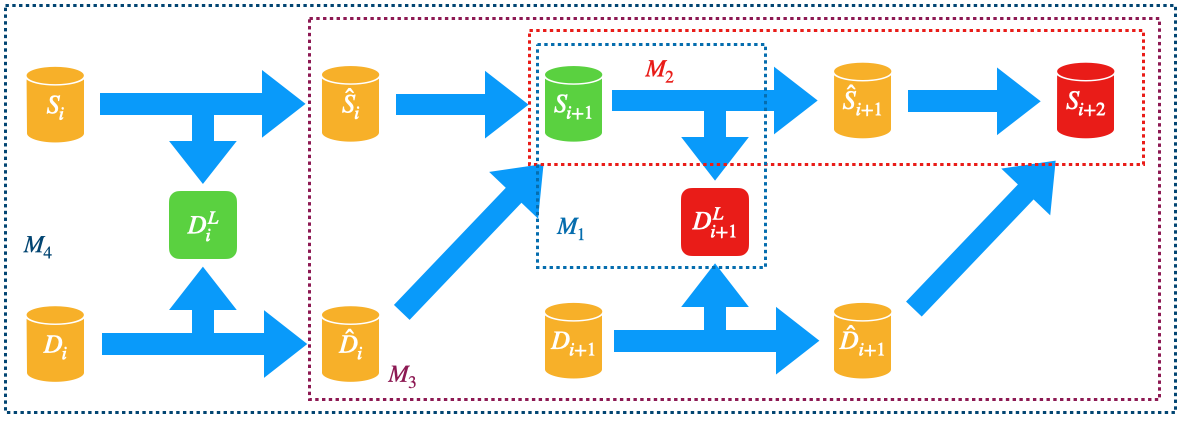}
    \caption{\label{fig:privacy_D}Privacy leakage of $D_i$}
\end{figure}

Let $D_1^0,D_2^0,\dots,D_k^0$ and $D_1^1,D_2^1,\dots,D_k^1$ be the set of unlabeled databases in step~\ref{GenericBBL_UnlabeledDatabases} of GenericBBL. Without loss of generality, we assume $D_i^0$ and $D_i^1$ differ on one entry. When $i=k$, $\mbox{View}^0_{{\cal B}, t} = \mbox{View}^1_{{\cal B}, t}$ because all selected hypothesis are the same. When $i<k$, let $R'=\min\left(\sum_{j=1}^{i+1}{R_j},t\right)$.

Similar to the analysis if Claim~\ref{clm:privacy-S}, $\mbox{View}^b_{{\cal B}, t}$ results from a post-processing of the random variable $(\mbox{View}^b_{{\cal B}, R'},S^b_{i+2})$ (if $t< \sum_{j=1}^{i+1}{R_j}$ we set $S^b_{i+2}$ as $\bot$). Note that $\mbox{View}^b_{{\cal B}, R'_1}=(D_1^{b,L},\dots,D_{i}^{b,L*},D_{i+1}^{b,L})$, and $(D_1^{b,L},\dots,D_{i-1}^{b,L},D_{i}^{b,L*})$ follow the same distribution for $b\in\{0,1\}$, where $D_{i}^{b,L*}$ is the labels of points in $D_i^b$ expect the different point. So that it suffices to show that $\left(D_{i+1}^{0,L}, S^0_2\right)$ and $\left(D_{i+1}^{1,L}, S^1_2\right)$ are $(\varepsilon,\delta)$-indistinguishable. 

We follow the processes creating $D_{i+1}^{b,L}$ and $S^b_{i+2}$ in Figure~\ref{fig:privacy_D}: 
(i) The mechanism $M_1$ corresponds to the loop in Step~\ref{GenericBBL_main_loop} of \texttt{GenericBBL} where labels are produced for the adversarially chosen points $D^b_{i+1}$.
By application of Lemma~\ref{lem:bt-privacy}, $M_1$ is $(1,\delta)$-differentially private.
(ii) The mechanism $M_2$, corresponds to the subsampling of $\hat{S}^b_{i+1}$ from $S^b_{i+1}$ and the application of procedure \texttt{LabelBoost} on the subsample in Step~\ref{GenericBBL_LabelBoost} of \texttt{GenericBBL} resulting in $S^b_{i+2}$. 
By application of Claim~\ref{claim:sub-sampling} and Lemma~\ref{lemma:LabelBoostPrivacy}, $M_2$ is $(\varepsilon,0)$-differentially private.
Thus $(M_1,M_2)$ is $(\varepsilon+1,\delta)$-differentially private. 
(iii) The mechanism $M_3$ with input of $\hat{D}_i^b$ and output $\left(D_{i+1}^{b,L},S^b_{i+2}\right)$ applies $(M_2,M_3)$ on $S_{i+1}$, which is generated from $\hat{D}_i^b$ and in Step~\ref{GenericBBL_LabelBoost} of \texttt{GenericBBL}. By application of Claim~\ref{lemma:LabelBoostPrivacy}, $M_3$ is $(\varepsilon+4,4\varepsilon\delta)$-differentially private. 
(iv) The mechanism $M_4$, corresponds to the subsampling $\hat{D}_i^b$ from $D^b_i$ and the application of $M_4$ on $\hat{D}_i^b$. By application of Claim~\ref{claim:sub-sampling}, $M_4$ is $(\varepsilon, \frac{16e\varepsilon\delta}{3+\mbox{exp}(\varepsilon+4)})$-differentially private. Since $\frac{16e\varepsilon}{3+\mbox{exp}(\varepsilon+4)}\leq 1$ for any $\varepsilon$,  $\left(D_{i+1}^{0,L}, S^0_2\right)$ and $\left( D_{i+1}^{1,L}, S^1_2\right)$ are $(\varepsilon,\delta)$-indistinguishable. 
\end{proof}

\begin{remark}
    The above proofs work on the adversarially selected $D$ because: 
    (i) Lemma~\ref{lem:bt-privacy} works on the adaptively selected queries. (We treat the hypothesis class $F_i$ as the database, the unlabelled points $x_{i,\ell}$ as the query parameters.)
    (ii) \texttt{LabelBoost} generates labels by applying one private hypothesis on points. The labels are differentially private by post-processing.
\end{remark}

\end{document}